\documentclass[review]{elsarticle}
 
\usepackage[margin=3cm]{geometry}
\usepackage{graphicx,url}

\usepackage{booktabs}   
\usepackage{siunitx}
\usepackage{makeidx}
\usepackage{subcaption}
\usepackage{array}
\usepackage{graphicx}
\usepackage{amssymb}
\usepackage{graphicx}
\usepackage{enumerate} 
\usepackage{amsmath} 
\usepackage{bigstrut}
\usepackage{amsopn}
\usepackage{algorithm}
\usepackage{algorithmicx}
\usepackage{algpseudocode}
\usepackage{float}
\usepackage{bm} 
\usepackage{rotating}
\usepackage{multicol}
\usepackage{multirow} 
\usepackage{tabularx}
\usepackage{arydshln}
\usepackage{color}
\usepackage{amsthm}

 \newtheorem{definition}{Definition}

\usepackage{hyperref}
 \hypersetup{
   colorlinks,
   allcolors=blue,
   urlcolor=blue,
}
 \usepackage{color}
\usepackage{blindtext}
\usepackage{scrextend}
\addtokomafont{labelinglabel}{\sffamily}

\usepackage[misc,geometry]{ifsym}  

\usepackage{lineno,hyperref}
\modulolinenumbers[5]

\DeclareMathOperator*{\argmax}{arg\,max}
\DeclareMathOperator*{\argmin}{arg\,min}

\newtheorem{thm}{Theorem}
\newtheorem{lem}{Lemma}

\journal{Information Systems}
\begin{document} 

\begin{frontmatter}

\title{Hierarchical clustering that takes advantage of \\both density-peak and density-connectivity}

\author[deakin]{Ye Zhu\corref{mycorrespondingauthor}}
\cortext[mycorrespondingauthor]{Corresponding author}
\ead{ye.zhu@ieee.org}
\author[kaiming]{Kai Ming Ting}
\ead{tingkm@nju.edu.cn}
\author[monash]{Yuan Jin} 
\ead{yuan.jin@monash.edu}
\author[deakin]{Maia Angelova}
\ead{maia.a@deakin.edu.au}
\address[deakin]{School of Information Technology, Deakin University, Victoria, Australia 3125}
\address[kaiming]{National Key Laboratory for Novel Software Technology, Nanjing University, Nanjing, China 210023}
\address[monash]{Faculty of Information Technology, Monash University, Victoria, Australia 3800}

\begin{abstract}
This paper focuses on density-based clustering, particularly the Density Peak (DP) algorithm and the one based on density-connectivity DBSCAN; and proposes a new method which takes advantage of the individual strengths of these two  methods to yield a density-based hierarchical clustering algorithm. We first formally define the types of clusters DP and DBSCAN are designed to detect; and then identify the kinds of distributions that DP and DBSCAN individually fail to detect all clusters in a dataset. These identified weaknesses inspire us to formally define a new kind of clusters and propose a new method called DC-HDP to overcome these weaknesses to identify clusters with arbitrary shapes and varied densities. In addition, the new method produces a richer clustering result in terms of hierarchy or dendrogram for a better understanding of cluster structures. Our empirical evaluation results show that DC-HDP produces the best clustering results on 28 datasets in comparison with 8  state-of-the-art clustering algorithms.
\end{abstract}

 
\end{frontmatter}

\section{Introduction}\label{sec:introduction}
 
Clustering is an important and useful tool in data mining and knowledge discovery. It has been widely used for partitioning instances in a dataset such that similar instances are grouped together to form a cluster \cite{han2011data}. It is the most common unsupervised knowledge discovery technique for automatic data-labelling in various areas, e.g., information retrieval, image segmentation, and pattern recognition \cite{kaufman2009finding}. Depending on the basis of categorisation, clustering methods can be divided into several kinds, e.g., partitioning clustering versus hierarchical clustering; and density-based clustering versus representative-based clustering \cite{zaki2014dataminingbook}.

Partitioning clustering methods are the simplest and most fundamental clustering methods \cite{han2011data}. They are relatively fast, and easy to understand and implement. They organise data points in a given dataset into $k$ non-overlapping partitions, where each partition represents a cluster; and each point belongs to one cluster only \cite{han2011data}. However, traditional distance-based  partitioning methods, such as $k$-means \cite{hartigan1979algorithm}  and $k$-medoids \cite{kaufman1987clustering}, which are representative-based clustering, usually cannot find clusters with arbitrary shapes \cite{aggarwal2013data}. In contrast, density-based clustering algorithms can find clusters with arbitrary sizes and shapes while effectively separating noise. Thus, this kind of clustering is attracting more research and development. 

DBSCAN \cite{ester1996density} and DENCLUE \cite{hinneburg2007denclue} are examples of an important class of density-based clustering algorithms. 
They define clusters as regions of high densities which are separated by regions of low densities. However, these algorithms have difficulty finding clusters with widely varied densities because a global density threshold is used to identify high-density regions \cite{ankerst1999optics,duan2007local,cassisi2013enhancing,HDBSCAN2013,ZHU2016983}. 

Rodriguez et al. proposed a clustering algorithm based on density peaks (DP) \cite{rodriguez2014clustering}. It identifies cluster modes\footnote{The original DP paper regards detected cluster modes as ``cluster centres'' \cite{rodriguez2014clustering}.} which have local maximum density and are well separated, and then assigns each remaining point in the dataset to a cluster mode via a linking scheme. Compared with the classic density-based clustering algorithms (e.g., DBSCAN and DENCLUE), DP has a better capability in detecting clusters with varied densities. Despite this improved capability, Chen et al. \cite{Chen2018} have recently identified a condition under which DP fails to detect all clusters with varied densities. They proposed a new measure called Local Contrast (LC) (instead of density) to enhance DP such that the resultant algorithm LC-DP is more robust against clusters with varied densities. 

It is important to note that the progression from DBSCAN or DENCLUE to DP, and subsequently LC-DP, with improved clustering performance, is achieved without formally defining the types of clusters DP and LC-DP can detect.

In this paper, we are motivated to formally define the type of clusters that an algorithm is designed to detect before investigating the weaknesses of the algorithm. This approach enables us to determine two additional weaknesses of DP; and we show that the use of LC does not overcome these weaknesses. 

This paper proposes a new clustering method which takes advantage of the individual strengths of DBSCAN and DP to yield a density-based hierarchical clustering algorithm that produces a better and richer clustering result. It makes the following contributions:

\begin{enumerate}[(i)] 
\item
Formalising a new type of clusters called $\eta$-linked clusters; and providing a necessary condition for a clustering algorithm to correctly detect all $\eta$-linked clusters in a dataset.
\item
Uncovering that DP is a clustering algorithm which is designed to detect $\eta$-linked clusters; and identifying two weaknesses of DP, i.e., the conditions under which DP cannot correctly detect all clusters in a dataset. 
\item
Introducing a different view of DP as a hierarchical procedure. 
Instead of producing flat clusters, this procedure generates a dendrogram, enabling a user to identify clusters in a hierarchical way. 
\item
Formalising the second new type of clusters called $\eta$-density-connected clusters which encompass all
$\eta$-linked clusters and the
kind of non-$\eta$-linked clusters that DP fails to detect.
\item
Proposing a density-connected hierarchical DP to overcome the identified weaknesses of DP. The new algorithm DC-HDP merges two cluster modes only if they are density-connected in the hierarchy.
\item 
Completing an empirical evaluation by comparing with 8 state-of-the-art clustering algorithms: 
 4 density-based clustering algorithms (DBSCAN \cite{ester1996density}, Mean shift clustering \cite{comaniciu2002mean},  DP \cite{rodriguez2014clustering} and LC-DP \cite{Chen2018}), 3 hierarchical clustering algorithms (OPTICS \cite{ankerst1999optics}, PHA \cite{LU20131227} and HDBSCAN \cite{Campello:2015:HDE}) and 1 graph-based spectral clustering algorithm \cite{Chen11}.
\end{enumerate}

The formal analysis of DP provides an insight into the key weaknesses of DP. This has enabled a simple and effective method (DC-HDP) to overcome the weaknesses. The proposed method takes advantage of the individual strengths of DBSCAN and DP, i.e., DC-HDP has an enhanced ability to identify all clusters of arbitrary shapes and varied densities; where neither DBSCAN nor DP has. In addition, the dendrogram generated by DC-HDP gives a richer information of  hierarchical components of clusters in a dataset than a flat partitioning provided by DBSCAN and DP. This is achieved with the same computational time complexity as in DP, having one additional parameter only which can usually be set to a default value in practice.

Since hierarchical clustering algorithms allow a user to choose a particular clustering granularity, hierarchical clustering is very popular and has been used far more than non-hierarchical clustering \cite{GILPIN201795}. Thus, DC-HDP provides a new perspective  which can be widely used in various applications. 

The rest of the paper is organised as follows: we provide an overview of density-based clustering algorithms and related work in Section \ref{sec_related_work}. Section \ref{sec_problem} formalises the $\eta$-linked clusters. Section \ref{sec:DP} uncovers that DP is an algorithm which detects $\eta$-linked clusters; and reveals two weaknesses of DP. Section \ref{DC} reiterates the definition of density-connected clusters used by DBSCAN, and states the known weakness of DBSCAN. Section \ref{DCC} presents the definition of the second new type of clusters called $\eta$-density-connected clusters. The new density-connected hierarchical clustering algorithm is proposed in Section~\ref{DCHDP}. In Section \ref{sec_result}, we empirically evaluate the performance of the proposed algorithm by comparing it with 8  other state-of-the-art clustering algorithms. Discussion and the conclusions are provided in the last two sections.

\section{Related work} 
\label{sec_related_work}

The two most representative algorithms of density-based clustering are DBSCAN \cite{ester1996density} and DENCLUE \cite{hinneburg2007denclue}. They first identify points in dense regions using a density estimator and then link neighbouring points in dense regions to form clusters. They can identify arbitrarily shaped clusters on noisy datasets. DBSCAN defines the density of a point as the number of points from the dataset that lie in its $\epsilon$-neighbourhood. A ``core'' point is a point having density higher than a threshold $MinPts$. DBSCAN visits every core point and links all core points in its $\epsilon$-neighbourhood together, until all core points are visited. Then, points which are directly/indirectly linked are grouped into the same cluster. Finally, non-core points that are in the $\epsilon$-neighbourhood of other core points, called boundary points, are linked to the nearest clusters. If a point is neither core point nor boundary point, then it is considered to be ``noise''. DENCLUE uses a Gaussian kernel estimator to estimate density for every point and applies a hill-climbing procedure to link neighbourhood points with high densities. Although DBSCAN and DENCLUE can detect clusters with varied sizes and shapes, they are sensitive to the parameter of density threshold \cite{huang2012revealing,bryant2018rnn}. In addition, they have difficulty finding clusters with widely varied densities because a global density threshold is used to identify points in high-density areas \cite{wang2009novel,ZHU2016983}.  

Many variants of DBSCAN have been attempted to overcome the weakness of detecting clusters with varied densities. OPTICS \cite{ankerst1999optics} draws a ``reachability'' plot based on  the ${K}^{th}$-nearest neighbour distance. 
In the $x$-axis of the plot, adjacent points follow close to each other such that point $p_{i}$ is the closest to $p_{i-1}$ in terms of the  ``reachability distance''\footnote{The ``reachability-distance'' of object $p_{i}$ to object $p_{i-1}$ is the greater one between the ``core distance'' of $p_{i}$ and the distance between $p_{i}$ and $p_{i-1}$. The ``core distance'' of $p_{i}$ is the minimum $\epsilon$ that makes $p_{i}$ a ``core'' object (the distance to its $\emph{K}^{th}$-nearest neighbour, $k=MinPts-1$).}. The reachability distance for each point is shown in $y$-axis. Since clusters centre normally has a higher density or lower reachability distance than the cluster boundaries, each cluster is visible as a ``valley'' in this plot. Then a hierarchical method can be used to extract different clusters. The overall clustering performance depends on the hierarchical method employed on the reachability plot.

HDBSCAN \cite{Campello:2015:HDE} is a hierarchical clustering based on DBSCAN. The idea is to produce many DBSCAN clustering outcomes through increasing density thresholds\footnote{It first builds a Minimum Spanning Tree (MST) for all points, where the weight of each edge is the mutual reachability distance and the weight for each vertex is the core distance. Then it removes edges from the MST progressively in decreasing order of weights, which is equivalent to getting many DBSCAN clustering outcomes with increasing density thresholds.}. As the density threshold increases, a cluster may shrink or be divided into smaller clusters or disappear. A dendrogram is built based on these clustering outcomes via a top-down method to yield the hierarchical cluster structure, e.g., the root is one cluster with all points and then split into different sub-clusters in following levels. To produce a flat partitioning, it extracts a set of `significant' clusters at different levels of the dendrogram via an optimisation process.  

HDBSCAN can detect clusters with varied densities because it employs different density thresholds. However, it has a bias towards low-density clusters. To separate overlapping high-density clusters, HDBSCAN needs to use a high-density threshold so that points in boundary region are treated as noise. As a result, high-density clusters can lose many cluster members when a high density level is used. 

Mean shift clustering \cite{comaniciu2002mean} is a mode-seeking algorithm based on gradient ascent, i.e.,  shifting points towards the local modes. For each point, it defines a window around it and calculate the mean of the points in the window.  Then it shifts the centre of the window to the mean and repeats this process till the algorithm stop updating.  All points shift to the same local maxima are grouped into the same cluster. The main disadvantage of mean shift is that it cannot control over the number of clusters. Furthermore, it cannot differentiate between meaningful and non-meaningful modes \cite{carreira2015review}.

\begin{figure}[!htb]
\centering
  \begin{subfigure}[b]{0.44\textwidth}
  \centering\captionsetup{width=.9\linewidth}%
    \includegraphics[width=\textwidth]{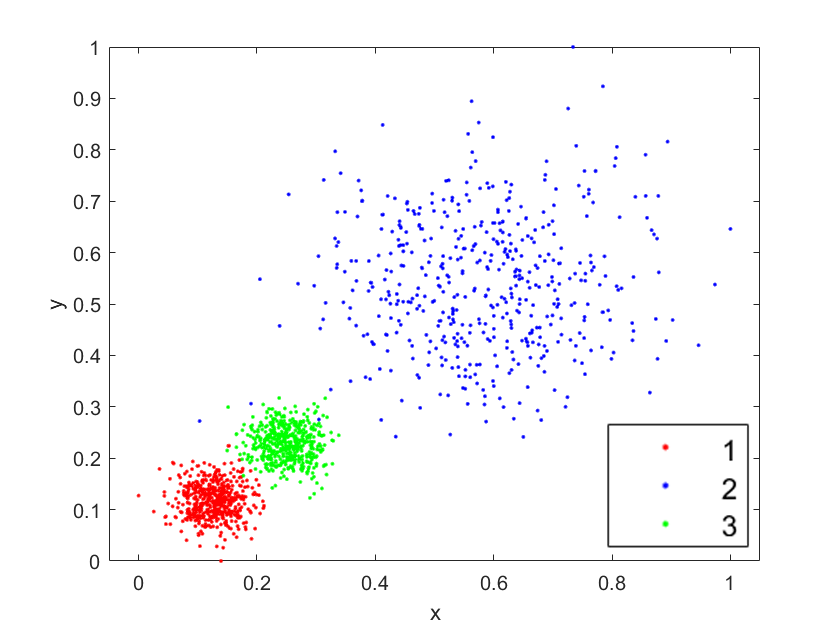}
    \caption{3G: A dataset with three Gaussians of different densities}
    \label{illu:a}
  \end{subfigure}  %
  \begin{subfigure}[b]{0.44\textwidth}
    \centering\captionsetup{width=.9\linewidth}%
    \includegraphics[width=\textwidth]{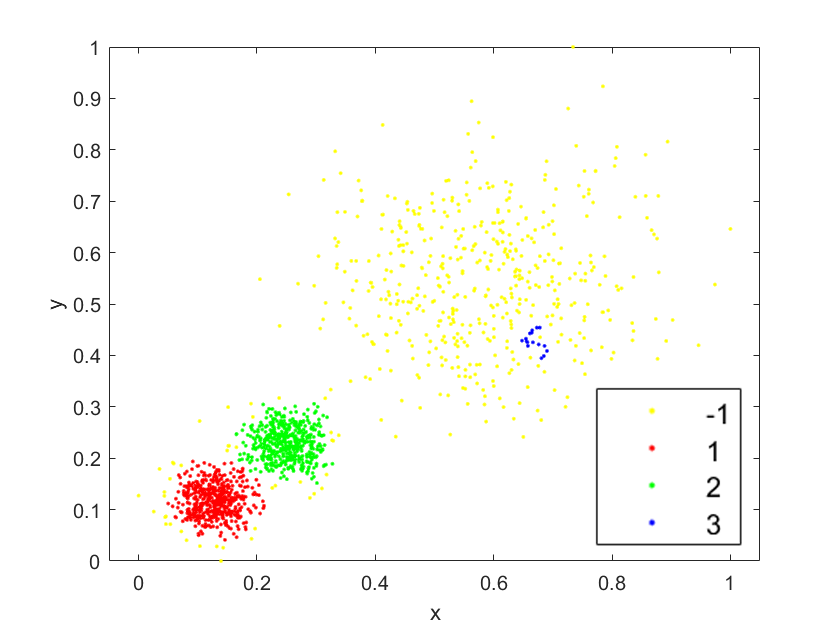}
    \caption{Clustering result of DBSCAN with the best F-measure of 0.67}
        \label{illu:b}
  \end{subfigure}\\
  \begin{subfigure}[b]{0.44\textwidth}
  \centering\captionsetup{width=.9\linewidth}%
    \includegraphics[width=\textwidth]{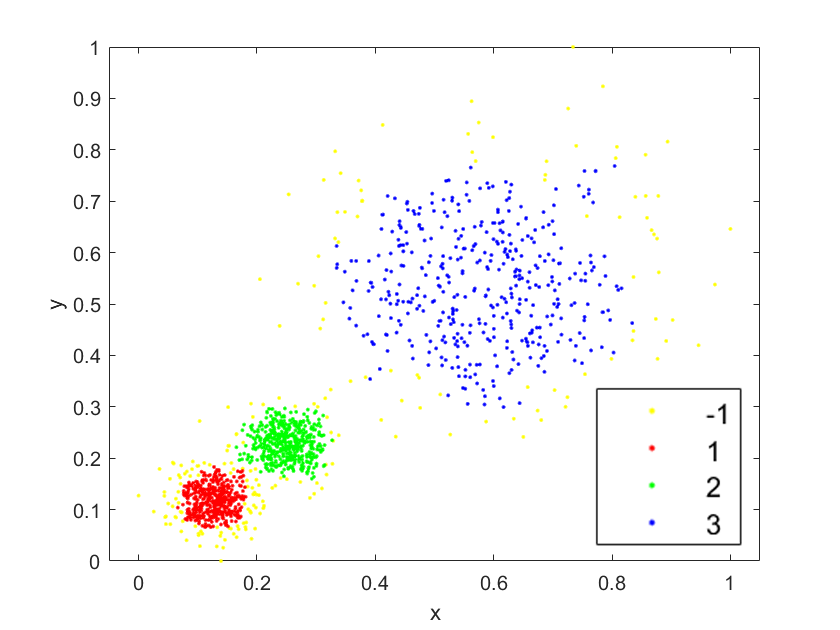}
    \caption{Clustering result of HDBSCAN with the best F-measure of 0.92}
    \label{illu:c}
  \end{subfigure}  %
  \begin{subfigure}[b]{0.44\textwidth}
    \centering\captionsetup{width=.9\linewidth}%
    \includegraphics[width=\textwidth]{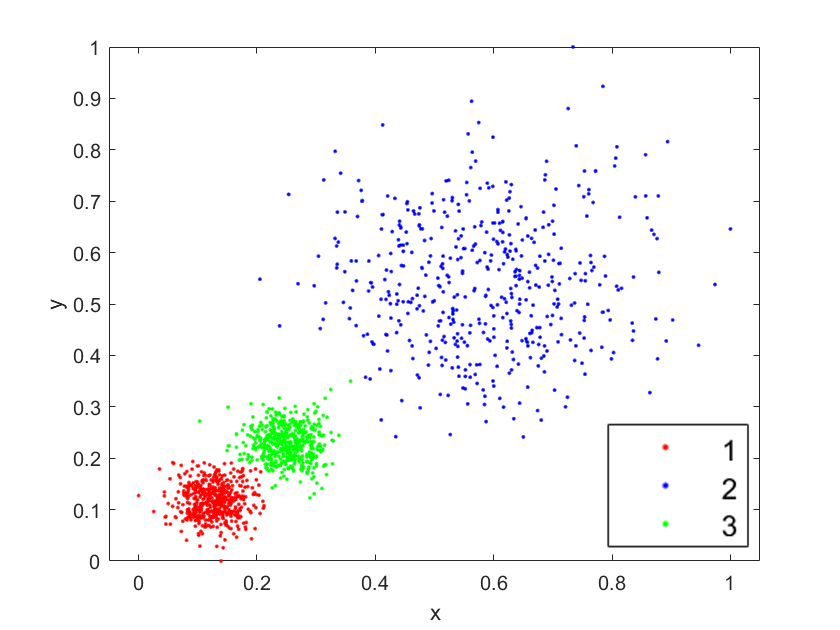}
    \caption{Clustering result of DP with the best F-measure of 0.99}
        \label{illu:d}
  \end{subfigure}
   \caption{Clustering results of different algorithms on a dataset having clusters with varied densities. ``-1'' indicates the noise assigned by the algorithm with light yellow colour. Note that in Figure (b), only the detected sparse cluster with the most assigned points is labelled as cluster 3.} 
    \label{illu}  
\end{figure}

Recently, density peak (DP) clustering algorithm was proposed without using any density threshold \cite{rodriguez2014clustering}. 
It assumes that cluster modes in the given dataset are 
sufficiently separated from each other. The clustering process has three steps as follows. First, DP calculates the density $\rho$ of each point using an $\epsilon$-neighbourhood density estimator, and the distance $\delta$ between a point and its nearest neighbour with a higher density value. Second, DP plots a decision graph for all points where the $y$-axis is $\rho \times \delta$, sorted in descending order in the $x$-axis. The points with the highest $\rho \times \delta$ (i.e., high density values and relatively far nearest neighbour with a higher density) are selected as the cluster modes. Third, each remaining point is connected to its nearest neighbour of higher density, and the points connected or transitively connected to the same cluster mode are grouped into the same cluster. Noise can be detected by applying an additional step.\footnote{Once every point is assigned to a cluster, a border region is identified for each cluster---it is the set of points which is assigned to that cluster but located within the $\epsilon$-neighbourhood of a point belonging to another cluster.  Let $\acute{\rho}$ be the highest density of all points within the border region.   Noise points are all points of the cluster which have density values less than $\acute{\rho}$.\label{foootnote_dp}} Since DP does not rely on any density threshold to extract clusters, it avoids the weakness of DBSCAN in detecting clusters with varied densities. Though both DP and mean shift clustering are mode seeking clustering, DP has clearer and meaningful definition of the mode and  mode selection criteria.     

Figure \ref{illu} illustrates clustering results of DBSCAN, HDBSCAN and DP on a dataset having clusters with varied densities. Only DP can near-perfectly detect all clusters in this dataset. DBSCAN performed poorly because it produced either a merged cluster of the two high-density clusters using a low-density threshold or only a few points are assigned to the low-density cluster when using a high density threshold. HDBSCAN can separate all three clusters but many boundary points of the high-density clusters are regarded as noise. This is because HDBSCAN has to use a high-density threshold to separate the two dense clusters and then the low-density points from the dense clusters become noise. 

DP is one of the most promising clustering algorithms which has the ability to detect clusters with varied densities and arbitrary shapes. Without formally articulating the weaknesses of DP, some papers have proposed to improve DP. For example, a fuzzy weighted $K$-nearest neighbour is used to more efficiently search for density mode \cite{XIE201619}; a linear fitting method is proposed to determine more meaningful cluster modes on the decision graph \cite{XU2016200};  an incremental method is designed to enable DP for clustering large dynamic data in industrial Internet of Things \cite{7882646}; and an adaptive $k$-nearest neighbour density estimator for accurate density estimation in a manifold space for high-dimensional data \cite{d2018automatic}. 
Only Chen et al. \cite{Chen2018} have formally identified one necessary condition under which DP fails to detect all clusters of varied densities in a dataset, and proposed a new measure called Local Contrast (LC) to make DP more robust against clusters of varied densities. 

However, to the best of our knowledge, there is no paper formalising the type of clusters that DP is designed to detect in a dataset. This knowledge is important to understand the weaknesses of DP and devise a way to overcome the weaknesses.

\section{Defining clusters based on higher-density-nearest-neighbours: \texorpdfstring{$\eta$-}{}linked clusters}
\label{sec_problem}
In this section, we formally formulate a type of clusters called $\eta$-linked clusters, where $\eta$ stands for nearest neighbour with higher density. The main symbols and notations used are listed in Table~\ref{tab:addlabel}. 

\begin{table}[!hbt]
	\centering
	\caption{Symbols and notations}
	\begin{tabular}{|ll|}
		\hline
		$x$     & a point in  $\mathbb{R}^d$ \\
		$D$    & a $d$-dimensional dataset with $n$ points \\
		$\hat{m}$     & a point with the highest density in $D$ \\
		$C$, $\widehat C$, $\widetilde C$, $\bar C$     & a cluster (a group of points) \\
		$D_{\ominus}$ & $D\setminus \{\hat{m}$\} \\
		$m$     & the mode (point of the highest density) in \\& a cluster \\
		$\rho(x)$ & density of $x$\\
		$\rho_{\epsilon}(x)$& the estimated density of $x$ via an \\& $\epsilon$-neighbourhood  estimator \\
		$dis(x,y)$   &the distance between two points $x$ and $y$ \\
		$\mathcal{N}_{\epsilon}(x)$ & $\epsilon$-neighbourhood of $x$\\
	    $\epsilon$ & radius of neighbourhood $\mathcal{N}_{\epsilon}$\\
		$\tau$   & density threshold \\
		$V_\epsilon^d$ & volume of a $d$-dimensional ball of radius $\epsilon$  \\
		$\eta_x$ &  $x$'s nearest neighbour which has a higher \\& density  than $x$\\
		$\eta_x'$ &  $x$'s nearest density-connected neighbour which \\& has a higher density  than $x$\\
		$path(x,y)$ &   an $\eta$-linked path connecting $x$ and $y$\\
		$DCpath_{\epsilon}^{\tau}(x, y) $  &   an $\eta$-density-connected  path connecting $x$ and $y$\\
		$Lpath(x,y)$ & length of $path(x,y)$ \\
		$LDCpath_{\epsilon}^{\tau}(x, y) $  & length of $DCpath_{\epsilon}^{\tau}(x, y) $\\
		$\delta(x)$ &the distance between $x$ and $\eta_x$ \\ 
		$\delta(x)'$ &the distance between $x$ and $\eta_x'$ \\ 
		$\gamma(x)$ &  $\rho(x)\times \delta(x)$ \\ 
		$\gamma(x)'$ &  $\rho(x)\times \delta(x)'$ \\ 
		\hline
	\end{tabular}%
	\label{tab:addlabel}%
\end{table}%

Let $ D=\lbrace x_{1}, x_{2}, ..., x_{n} \rbrace$, $x_{i}\in \mathbb R^{d}, x_{i} \sim F$ denote a dataset of $n$ points, each sampled independently and identically from a distribution $F$. 

Let $\mathcal{N}_\epsilon(x)$ be the $\epsilon$-neighbourhood of $x$, $\mathcal{N}_\epsilon(x)=\lbrace y \in D ~|~ dis(x,y) \leqslant \epsilon \rbrace$, where $dis(\cdot,\cdot)$ is the distance function ($s:\mathbb{R}^{d}\times \mathbb{R}^{d}\rightarrow \mathbb{R}$); and $\epsilon$ is a user-defined constant.
 
In general, the true density of $x$, i.e., ${\rho}(x)$, can be estimated via an $\epsilon$-neighbourhood estimator (as used by density-based clustering algorithms DBSCAN \cite{ester1996density} and DP \cite{rodriguez2014clustering}):

\begin{equation}
	 \rho_{\epsilon}(x)=\frac{1}{nV_{\epsilon}^d} \vert  \mathcal{N}_\epsilon(x) \vert=\frac{\vert \lbrace y \in D ~|~ dis(x,y) \leqslant \epsilon \rbrace \vert}{nV_{\epsilon}^d}
	\label{epsN}
\end{equation}

\noindent where $V_{\epsilon}^d \propto \epsilon^d$ is the volume of a $d$-dimensional sphere of radius $\epsilon$, and $n$ is the number of data points. Note that since $nV_{\epsilon}$ is a constant for every point, it can be omitted in practice. As a result, many density-based clustering algorithms use the number of points in the $\epsilon$-neighbourhood as a proxy to the estimated density.

Let $\hat{m} = \argmax_{{x} \in D} \rho( x)$ denote the point with the global maximum density; and $D_{\ominus} = D  \setminus \{\hat{m}\}$. For each point $x\in D_{\ominus}$; and let $\eta_x$ be $x$'s nearest neighbour which has a higher density than $x$, i.e., 

\begin{equation}
\displaystyle \eta_x = \argmin\limits_{y \in D, \ \rho(y) > \rho(x)} dis(x, y)
\label{eta_x}
\end{equation}

Assuming we have $k$ cluster modes $M=\{m_i, i=1,...,k\}$ (which will be defined in Definition \ref{def:cluster}). Each of $n-k$ points in $D$ is then assigned to the same cluster of its nearest neighbour of higher density; and the points, which are path-linked or transitively path-linked to the same closest cluster mode in terms of their path length, are grouped into the same cluster. Such a path is defined as follows:

\begin{definition} 
An $\eta$-linked path connecting points $x_{1}$ and $x_{p}$: $path(x_{1}, x_{p})=\{x_{1},x_{2}, x_{3},\dots,x_{p}\}$ is defined as a sequence of the smallest number of $p$ unique points starting with $x_{1}$ and ending with $x_{p}$, where $\forall i \in \{1,\dots,p-1\}, x_{i+1}=\eta_{x_{i}} = \argmin\limits_{y \in D, \ \rho(y) > \rho(x_i)} dis(x_i, y)$.
\label{path}
\end{definition}

\begin{definition} 
The length of path(x,y) is defined as 
\begin{equation}
    Lpath(x,y) = \left\{ \begin{array}{l}    
         (i) \ \ \vert path(x, y) \vert,  \text{if there exists an $\eta$-linked} \\ \qquad  \qquad \qquad \ \ \ \ \   \text{path connecting x and y} \\  
         (ii) \ \infty, \qquad \qquad   otherwise
    \end{array} \right.
\label{length}
\end{equation}
where $\vert path(x, y) \vert$ is the number of points along the path from $x$ to $y$. 

Note that $\vert path(x, y) \vert =1 $ if $x=y$; and $\vert path(x, y) \vert > 1 $, if $x \ne y$.
\end{definition}


\begin{definition}
A set of points $\widehat C_i$ is an $\eta$-linked cluster relative to a choice of  cluster modes, that is, $\widehat C_i$ is a maximal set of points having the shortest path length to its cluster mode $m_i$ wrt other cluster modes $m_j$, i.e., $\widehat C_i=\{ x\in D \ | \  \forall_{m_j\neq m_i}  \ Lpath(x,m_j)  > Lpath(x,m_i)  \}$, where $m_{i}=\arg\max_{\substack{x\in \widehat C_i}}{\rho}(x)$.
\label{def:cluster}
\end{definition}

As a result of the above definitions, three lemmas are stated as follows:

\begin{lem}
Every point in $D_{\ominus}$ has a path to $\hat{m}$ when the dataset only has one mode, i.e., $\forall_{x\in D_{\ominus}} \ 1 < Lpath(x, \hat{m}) \leqslant n $. 
\label{lem:1}
\end{lem}

\begin{proof}
Since $\forall_{x\in D_{\ominus}}, \rho(\hat{m}) > \rho(x)$, $x$ will link or transitively link to $\hat{m}$ by an $\eta$-linked path. If the path include all data points, then $Lpath(x, \hat{m})=n$; if $\eta_x=\hat{m}$, then $Lpath(x, \hat{m})=1$. Thus, $1 < Lpath(x, \hat{m}) \leqslant n $
\end{proof} 
  
\begin{lem}
Given a dataset with non-overlapping clusters, i.e., each instance belongs to only one cluster, if the density at every cluster $C_i$ is a strictly  decreasing function of distance to the mode $m_i$ on any directions in $C_i$, i.e., if
$ \forall_{x,y \neq m_i  \in C_i, cos(<m_i,x>,<m_i,y>)=1}\   dis(x,m_i)   > dis(y,m_i) \to \rho(x) < \rho(y) < \rho(m_i)$, then every cluster $C_i$ is an $\eta$-linked cluster, where $cos(\cdot,\cdot)$ is the cosine similarity. 
\label{mono}
\end{lem}

\begin{proof}
When $(\forall_{z\neq m_i \in C_i}\ \rho(z) < \rho(m_i))$, there is a single cluster mode for each cluster. When $x$ and $y$ lie on the same direction from the mode $m_i$, i.e., $cos(<m_i,x>,<m_i,y>)=1$,  and  $(\forall_{x,y \in C_i}\ dis(x,m_i) > dis(y,m_i) \to \rho(x) < \rho(y))$, each point in the cluster has an $\eta$-path direction toward the cluster mode and will link the cluster mode at the end. Therefore, every cluster is an $\eta$-linked cluster.
\end{proof} 

The strictly decreasing function in Lemma \ref{mono} implies (1) an $\eta$-linked cluster has only one cluster mode with the highest density over all cluster instances; (2) the density of the cluster instance is strictly decreasing wrt the distance between the cluster instance and the cluster mode given any direction from the mode. 

Figure \ref{illud:a} illustrates an example of two clusters having strictly decreasing functions. After selecting Peak 1 and Peak 2 as cluster modes and assigning the rest of the points based on Definition~\ref{def:cluster}, the red dash-line indicates the boundary of the two clusters.

\begin{figure}[!htbp]
\centering
  \begin{subfigure}[b]{0.44\textwidth}
  \centering\captionsetup{width=.9\linewidth}%
    \includegraphics[width=\textwidth]{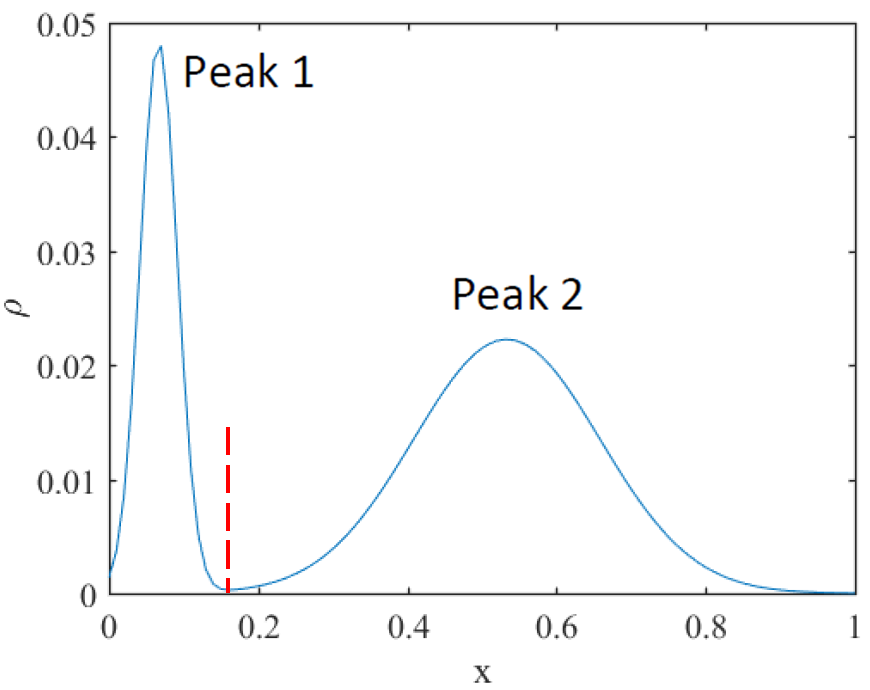}
    \caption{A distribution with two local peaks}
    \label{illud:a}
  \end{subfigure}  %
  \begin{subfigure}[b]{0.44\textwidth}
  \centering\captionsetup{width=.9\linewidth}%
    \includegraphics[width=\textwidth]{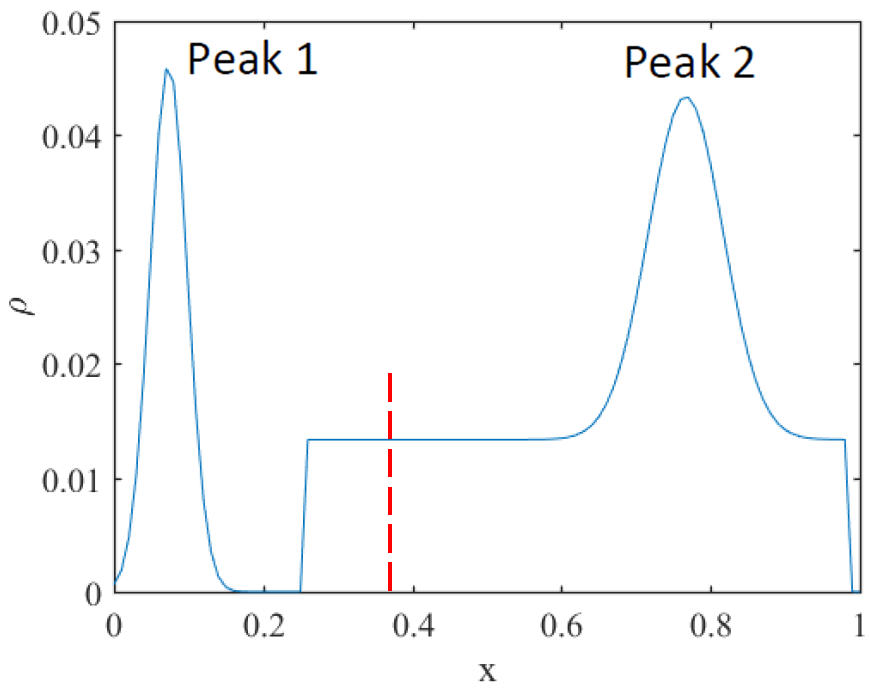}
    \caption{A distribution with two local peaks}
    \label{illud:b} 
  \end{subfigure} \\
  \begin{subfigure}[b]{0.44\textwidth}
    \centering\captionsetup{width=.9\linewidth}%
    \includegraphics[width=\textwidth]{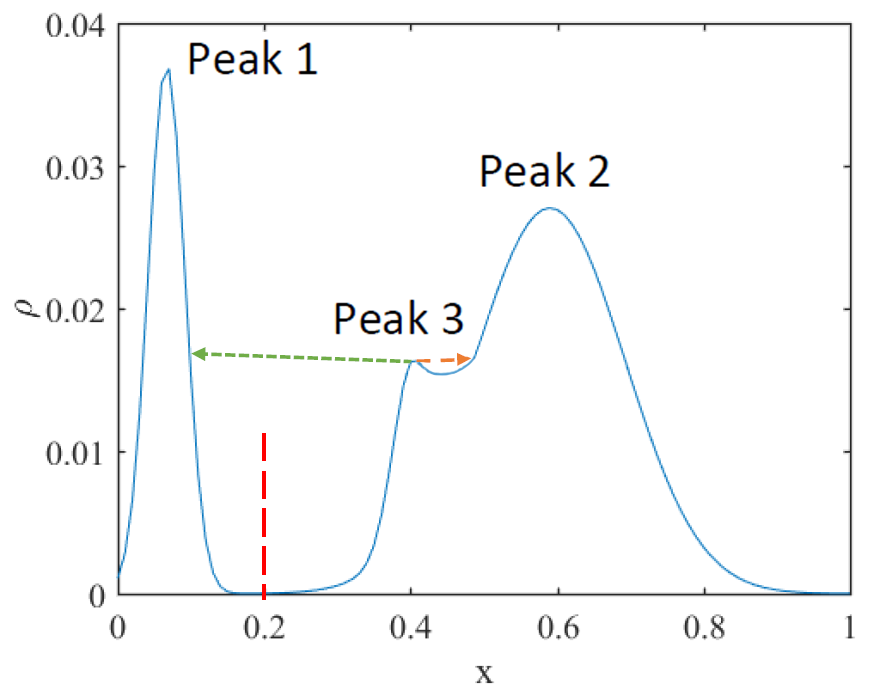}
    \caption{A distribution with three local peaks}
        \label{illud:c}
  \end{subfigure}
      \begin{subfigure}[b]{0.44\textwidth}
        \centering\captionsetup{width=.9\linewidth}%
    \includegraphics[width=\textwidth]{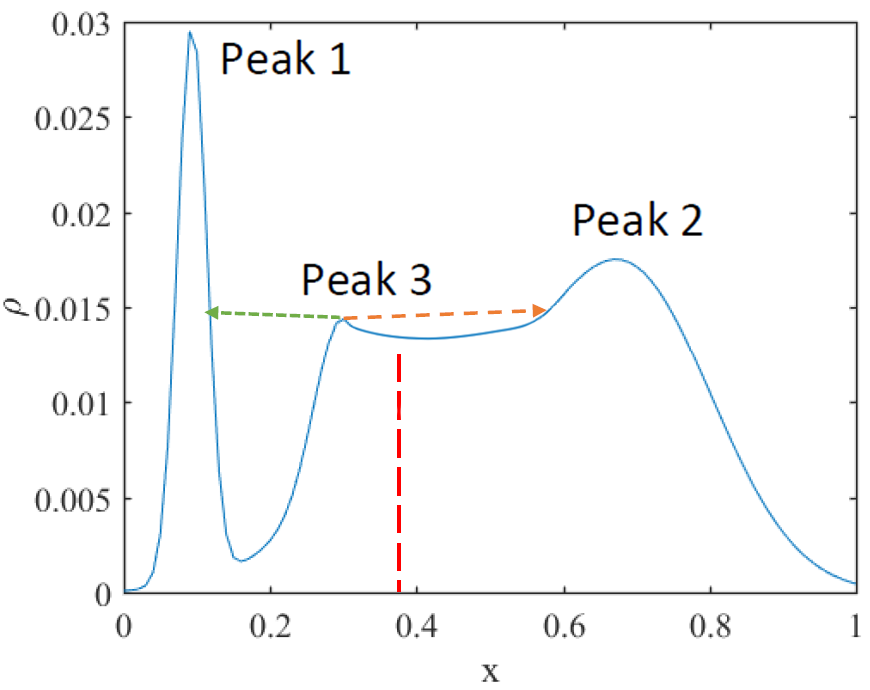}
    \caption{A distribution with three local peaks}
        \label{illud:d}
  \end{subfigure}
   \caption{Density distribution of four one-dimensional datasets. Each dataset has two clusters with different densities. (a) only contains $\eta$-linked clusters but (b), (c) and (d)  contain non-$\eta$-linked cluster clusters. In each cluster of (a), the density of instance is strictly decreasing from the cluster mode on both left and right sides. In (c) and (d), the green and orange dash-lines point to the nearest neighbour with higher density of Peak 3 from left side and right side, respectively. The red dash-line indicates the boundary of the two clusters identified by DP. }
    \label{illud} 
\end{figure}

Given a dataset with non-overlapping clusters, if the density distribution of at least one cluster $C_i$ is a non-strictly decreasing function having its only mode at $m_i$, i.e., 
$(\forall_{z\neq m_i \in C_i}\ \rho(z) < \rho(m_i)) \wedge (\forall_{x,y \in C_{i}}\ dis(x,m_i) > dis(y,m_i) \to \rho(x) \le \rho(y))$,  then $C_i$ could be non-$\eta$-linked cluster. 

This occurs when not all $x \in C_i$ has the shortest $Lpath(x,m_i)$ wrt another cluster modes. An example is shown in Figure~\ref{illud:b}, where some points to the left of the boundary are assigned to the left cluster, though they should belong to the right cluster. 
A dataset, which has clusters with uniform density distribution only, does not contain $\eta$-linked clusters for the same reason. 

In practice, a dataset having clusters with slow rates of density decrease from the modes, even if every cluster is a strictly decreasing function, can raise the same issue because a density estimator may identify multiple local peaks for such a cluster due to estimation errors. 

In addition, a dataset having clusters with multiple local peaks may not contain $\eta$-linked clusters when the local peaks do not have the shortest path to the correct modes. Figures~\ref{illud:c} and \ref{illud:d} illustrate two example density distributions with two clusters having a local peak (Peak 3) in addition to the cluster mode (Peak 2). Peak 3 in Figure~\ref{illud:c} has the shortest path to Peak 2; and it is assigned to the same cluster with Peak 2. However, Peak 3 in Figure~\ref{illud:d} has the shortest path to Peak 1. This is because its nearest point with higher density is the point which has the shortest path to Peak 1. Note that this occurs even if Peak 3 is just a tiny bump in the data distribution. 

Therefore, {\em $\eta$-linked clusters are sensitive to (i) the distance separation between clusters; and (ii) the data density distribution}. 

Based on the first two lemmas and Definition~\ref{def:cluster}, we formulate a theorem for an algorithm which identifies $\eta$-linked clusters as follows:

\begin{thm}
Given a dataset $D$ with $k$ $\eta$-linked clusters, a necessary condition for a clustering algorithm, which connect points following the $\eta$-linked path, to correctly identify all $k$ clusters is that it must correctly select the mode in each and every cluster.
\label{thm:DP}
\end{thm}

\begin{proof}

A violation of the condition means at least one of the two following situations will occur:
\begin{enumerate}
\item
If a point, other than the true mode, is selected by the algorithm as a cluster mode, then the cluster will be split into at least two clusters. This is because all cluster members with higher density than the selected cluster mode do not have the shortest path to this mode and will be assigned to other cluster(s). 
\item
If no point in a cluster is selected as a cluster mode, then each member of this cluster will be linked to a different cluster with a mode having the shortest path from this member. 
\end{enumerate}
Therefore, in either situation, not all clusters are identified correctly by this clustering algorithm.
\end{proof} 

\section{DP is an algorithm which identifies  \texorpdfstring{$\eta$-}{}linked clusters}
\label{sec:DP}

Density Peak or DP \cite{rodriguez2014clustering} has two main procedures as follows:
\begin{enumerate}
\item
Identifying cluster modes via a ranking scheme which aims to rank all points. The top $k$ points are selected as the modes of $k$ clusters.
\item
Linking each non-mode data point to its nearest neighbour with higher density. The points directly linked or transitively linked to the same cluster mode are assigned to the same cluster. This produces $k$ clusters at the end of the process.
\end{enumerate}

Therefore, DP is an algorithm implementing Definition \ref{def:cluster} to identify $\eta$-linked clusters in a dataset (in step 2 above).

The first step is critical, as specified in Theorem \ref{thm:DP}.  To effectively identify cluster modes, DP  assumes that different cluster modes should have a relatively large distance between them in order to detect well-separated clusters \cite{rodriguez2014clustering}. To prevent a cluster from breaking into multiple clusters when it has a slow rate of density decrease from the cluster mode, it applies a ranking scheme on all points, and then selects the top $k$ points as cluster modes. This is done as follows.

DP defines a distance function $\delta(x)$ as follows:

\begin{equation}
    \delta(x) = \left\{ \begin{array}{l}    
        (i) \ \ dis(x, \eta_x), \ \ \ \ \ \  \ \forall x  \in D_{\ominus}\\  
        (ii) \ \displaystyle\max_{y \in D} dis(x, y), \text{ if } x = \hat{m}
    \end{array} \right.
\label{delta}
\end{equation}

DP selects the top $k$ points with the highest $\gamma(x)=\rho(x) \times \delta(x)$ as cluster modes. This means that each cluster mode should have high density and be far away from other cluster modes. 

\subsection{Weaknesses of DP}

While DP generally produces a better clustering result than DBSCAN  (see the evaluation conducted by Chen et al \cite{Chen2018} reported in Appendix C of the paper), we identify two fundamental weaknesses of DP:

\begin{enumerate}[(i)] 
\item
Given a dataset of $k$ $\eta$-linked clusters, if the data distribution is such that the $k$ cluster modes are not ranked as the top $k$ points with the highest $\gamma$ values, then DP cannot correctly identify these cluster modes, as stated in Theorem~\ref{thm:DP}. The source of this weakness is the ranking scheme in step 1 of the DP procedure.

An example is a dataset having two Gaussian clusters and an elongated cluster with two local peaks (the left one is the cluster mode), as shown in Figure \ref{fig01}. DP with $k=3$ splits the elongated cluster into two sub-clusters because  the two local peaks are ranked among the top 3 points with the highest $\gamma$ values; and it merges the bottom two clusters into one, as shown in Figure \ref{fig01:a}. A better clustering result can be obtained by using $k=4$ which resulted in a correct identification of the two bottom clusters, as shown in  Figure \ref{fig01:b}. But it still splits the single top cluster into two. Note that all three clusters would be correctly identified by DP
if the three true cluster modes were pre-selected for DP using a different process. This data distribution is similar to that shown in Figure \ref{illud:c} which has valid $\eta$-linked clusters.

\begin{figure}[!htb]
  \begin{subfigure}[b]{0.44\textwidth}
  \centering\captionsetup{width=.9\linewidth}%
    \includegraphics[width=\textwidth]{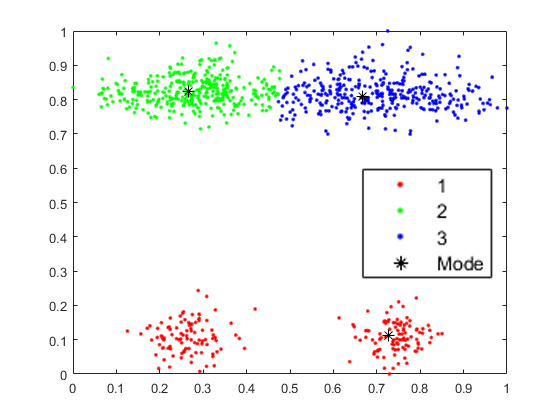}
    \caption{DP with $k=3$ yields the best clustering result by setting $\epsilon=0.1$, where  F-measure = 0.44}
    \label{fig01:a}
  \end{subfigure}  %
  \begin{subfigure}[b]{0.44\textwidth}
    \centering\captionsetup{width=.9\linewidth}%
    \includegraphics[width=\textwidth]{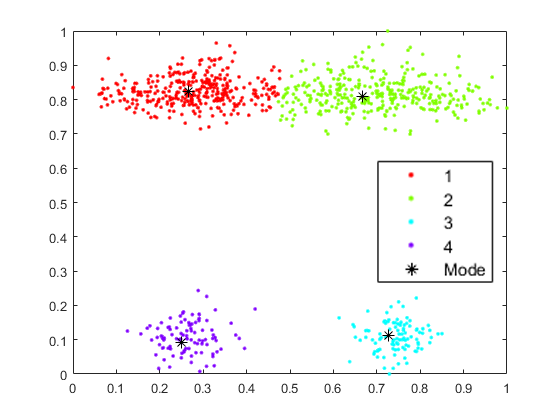}
    \caption{DP with $k=4$ yields the best clustering result by setting $\epsilon=0.1$, where  F-measure = 0.89}
        \label{fig01:b}
  \end{subfigure}
   \caption{Clustering result of DP on 3C dataset (a dataset with three clusters). The colours used in each scatter plot indicate clusters labelled by DP.}
    \label{fig01} 
\end{figure}

\begin{figure}[!htb]
  \begin{subfigure}[b]{0.44\textwidth}
  \centering\captionsetup{width=1\linewidth}%
    \includegraphics[width=\textwidth]{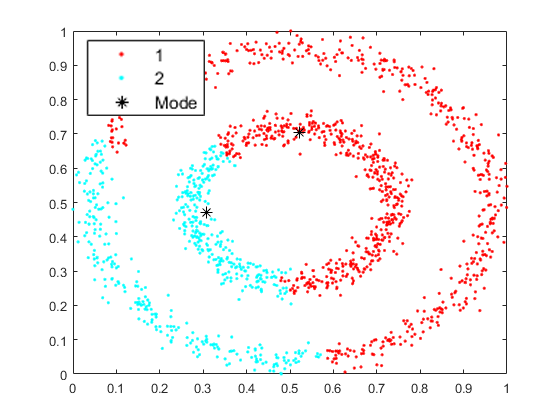}
    \caption{2O: a dataset with two circle-clusters}
    \label{fig1:a}
  \end{subfigure}  %
  \begin{subfigure}[b]{0.44\textwidth}
    \centering\captionsetup{width=1\linewidth}%
    \includegraphics[width=\textwidth]{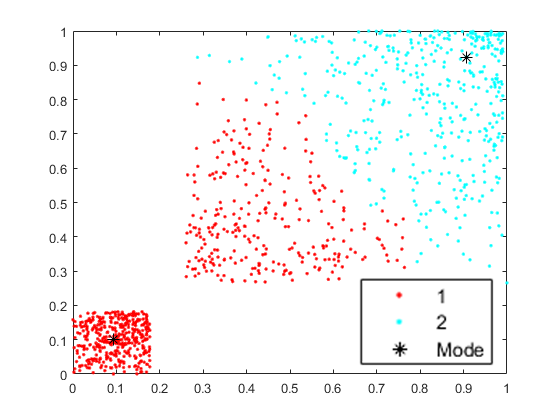}
    \caption{2Q: a dataset with two square-shaped clusters}
        \label{fig1:b}
  \end{subfigure}
   \caption{Clustering results of DP on two datasets (2O and 2Q), when $\epsilon=0.1$ and $k$ is the same as the true number of clusters. The colours used in each scatter plot indicate clusters labelled by DP.}
    \label{fig1} 
\end{figure}

\item
Given a dataset of non-$\eta$-linked clusters, DP cannot correctly identify these clusters because DP can detect $\eta$-linked clusters only. The source of this weakness is the limitation of the cluster definition used.

Figure \ref{fig1} demonstrates two examples when DP fails to correctly detect the clusters. These clusters are non-$\eta$-linked clusters, i.e., some points of a cluster do not have an $\eta$-linked path to its cluster mode, but have an $\eta$-linked path to the mode of another cluster;  even though the two clusters are well separated. Specifically, the 2O dataset in Figure \ref{fig1:a} has clusters with uniform distribution; and the data distribution of the 2Q dataset in Figure \ref{fig1:b} has small bumps on the side of the large cluster near the smaller cluster---a similar data characteristic shown in Figure \ref{illud:d}.

\end{enumerate}

\subsection{An existing improvement of DP}

Chen et al. \cite{Chen2018} provides a method called Local Contrast (LC-DP), which aims  to improve the ranking of cluster modes for detecting all clusters in a dataset with clusters of varied densities---the condition under which DP fails to correctly all clusters, i.e., the condition they have discovered.  

LC-DP uses local contrast $LC(x)$, instead of density ${\rho}(x)$, in Equation \ref{eta_x} to determine $\eta_x$. $LC(x)$ is defined to be the number of times that $x$ has a higher density than its $K$-nearest neighbours, which has values between 0 and $k$. Then the ranking is based on $LC(x)\times \delta_{LC}(x)$, where $\delta_{LC}(x)$ is the version of $\delta(x)$ based on $\eta_x$ which is determined by $LC(x)$. Chen et al. \cite{Chen2018} show that the use of LC makes DP to be more robust to clusters with varied densities.

LC-DP has the ability to enhance DP's clustering performance on clusters with varied densities \cite{Chen2018}, e.g., the 2Q and 3C datasets.
However, LC doesn't overcome the two weaknesses of DP mentioned above. For example, LC-DP still fails to identify all clusters on the 2O dataset which does not have $\eta$-linked clusters. Therefore, it is important to design a method to overcome DP's two weaknesses. Here we propose a hierarchical method based on density-connectivity with this aim in mind.

We first reiterate the currently known definitions of density connectivity and density-connected clusters in Section~\ref{DC}. Then, we define a new type of clusters based on density connectivity in the following section.

\section{Density-connected clusters}
\label{DC}

The classic density-based clustering, such as DBSCAN \cite{ester1996density}, defines a cluster based on density connectivity as follows:

\begin{definition}
Using an $\epsilon$-neighbourhood density estimator $ \rho_{\epsilon}(\cdot)$ with density threshold $\tau$,
a point $x_1$ is density connected with another point $x_p$ in a sequence of $p$ unique points from $D$, i.e., $\{x_1,x_2,x_3,...,x_p\}$:  $Connect_{\epsilon}^{\tau}(x_1, x_p)$  is defined as: 
 
\begin{equation}
\small
\begin{aligned}
    Connect_{\epsilon}^{\tau}(x_1, x_p) \leftrightarrow   \left\{ \begin{array}{l} (i) \ \text{if $p>2$:} \ \  \exists_{\{x_1,x_2,...,x_p\}} \ x_1 \triangleleft x_2 ... \triangleright x_p  \\
     (ii) \text{if $p=2$:} \ \ x_1 \triangleright x_p \vee x_1 \triangleleft x_p  
    \end{array} \right. 
\end{aligned}
\end{equation}
\noindent where $x_1 \triangleright x_p$ iff $dis(x_1,x_p)\leq \epsilon$ and $(\rho_{\epsilon}(x_1)\geq \tau)$.
\label{def:connect}
\end{definition}

\begin{definition}
A density-connected cluster $\widetilde C$, which has a mode $m=\arg\max_{\substack{x\in \widetilde C}}{\rho}(x)$, is a maximal set of points that are density connected with its mode, i.e., $\widetilde C=\{x\in D \ | \ Connect_{\epsilon}^{\tau}(x, m)\}$.  
\end{definition}

Based on the density-connectivity, we have the property that points in a density-connected cluster $\widetilde C$ are density connected to each other via the mode $m$, i.e., $\forall_{x \in \widetilde C} \ Connect_{\epsilon}^{\tau}(x, m)$.

Note that a set of points having multiple modes (of the same peak density) must be density connected together in order to form a density-connected cluster.

The key characteristic of a density-connected cluster is that the cluster can have an arbitrary shape and size \cite{ester1996density}. Although an $\eta$-linked cluster can have arbitrary shape and size as well, DP which detects $\eta$-linked clusters  has issues with the two types of data distributions, mentioned in Section \ref{sec:DP}. The $\eta$-linked path may link points from different clusters which are separated by low density regions, e.g., the two circle-clusters in Figure \ref{fig1:a}.

On the other hand, though a clustering algorithm such as DBSCAN which detects density-connected clusters does not have the above issues, DBSCAN has issues in identifying all clusters of varying densities \cite{ZHU2016983}, as shown in Figure \ref{illu:b} in Section \ref{sec_related_work}. 

In a nutshell, both the clustering algorithms designed for detecting $\eta$-linked clusters and the density-connected clusters have different limitations.

\section{\texorpdfstring{$\eta$-}{}density-connected clusters}
\label{DCC}

To overcome the limitations of (i) $\eta$-linked clusters stated in Sections \ref{sec_problem} and  \ref{sec:DP}; and (ii) density-connected clusters stated in Section \ref{DC}, we strengthen the $\eta$-linked path based on the density connectivity as follows:

\begin{definition} 
An $\eta$-density-connected  path linking points $x_{1}$ and $x_{p}$, $DCpath_{\epsilon}^{\tau}(x_{1}, x_{p})=\{x_1,x_{2},x_{3}..,x_{p}\}$, is defined as a sequence of the smallest number of $p$ unique points starting with $x_{1}$ and ending with $x_{p}$ 
such that $\forall_{i\in{\{1,...,p-1\}}} \ x_{i+1}= \eta'_{x_{i}}$, where $\eta_x'$ is $x$'s nearest density-connected neighbour which has a higher density than $x$, i.e.,
\begin{equation}
\eta'_{x}=\argmin\limits_{y \in D, \ \rho(y) > \rho(x),\  Connect_{\epsilon}^{\tau}(x,y)} dis(x, y)
\label{eta'}
\end{equation}
\end{definition}

\begin{definition} 
The length of $DCpath_{\epsilon}^{\tau}(x,y)$ is defined as 
\begin{equation}
    LDCpath_{\epsilon}^{\tau}(x, y) = \left\{ \begin{array}{l}    
        (i) \ \ \vert DCpath_{\epsilon}^{\tau}(x, y) \vert,  \text{if there exists a}\\ \ \quad \text{ $DCpath_{\epsilon}^{\tau}$ linking $x$ and $y$} \\  
        (ii) \ \infty, \qquad  \qquad \qquad  otherwise
    \end{array} \right.
\label{length2}
\end{equation} 
\end{definition} 
Note that $\vert DCpath_{\epsilon}^{\tau}(x, y) \vert =1 $ if $x=y$ and $\vert DCpath_{\epsilon}^{\tau}(x, y) \vert > 1 $, if $x \ne y$. 


\begin{definition}
 An $\eta$-density-connected cluster $\bar C_i$, which has only one mode $m_{i}=\arg\max_{\substack{x\in \bar C_{i}}}{\rho}(x)$, is a maximal set of density-connected points having the shortest density-connected path to its cluster mode $m_i$ wrt other cluster modes in terms of the path length, i.e., $\bar C_i=\{ x\in D \ | \  Connect_{\epsilon}^{\tau}(x, m_i) \wedge \forall_{m_j\neq m_i}  \ LDCpath_{\epsilon}^{\tau}(x,m_j)  > LDCpath_{\epsilon}^{\tau}(x,m_i) \}$.
 \label{def:DCHDP}
\end{definition}


Based on these definitions, we have that an $\eta$-linked cluster becomes an $\eta$-density-connected cluster if all points in the $\eta$-linked cluster are density-connected. In addition, if a dataset has only density-connected clusters and each cluster has only one mode, then all clusters are $\eta$-density-connected clusters.

It is worth mentioning that an $\eta$-linked cluster can be a density-connected cluster, providing the all $\eta$-linked paths in the cluster are $\eta$-density-connected path, i.e, all points are density-connected to the cluster mode. A density-connected cluster can be an $\eta$-linked cluster, providing each point in the cluster has an $\eta$-linked path to the cluster mode.

With proper neighbourhood threshold $\epsilon$ and density threshold $\tau$, well-separated clusters cannot be linked together as an $\eta$-density-connected cluster. This enables us to identify  clusters which are density-connected but are not $\eta$-linked in a dataset. Figure \ref{illud2} illustrates the cluster boundaries of two clusters after selecting Peak 1 and Peak 2 as cluster modes and assigning the rest of the points based on Definition \ref{def:DCHDP}. It can be seen that all these clusters can be identified as $\eta$-density-connected clusters now, although the clusters in Figure \ref{illud2:c} and Figure \ref{illud2:d} are not $\eta$-linked clusters. 

 \begin{figure}[!htb]
  \begin{subfigure}[b]{0.44\textwidth}
  \centering\captionsetup{width=.9\linewidth}%
    \includegraphics[width=\textwidth]{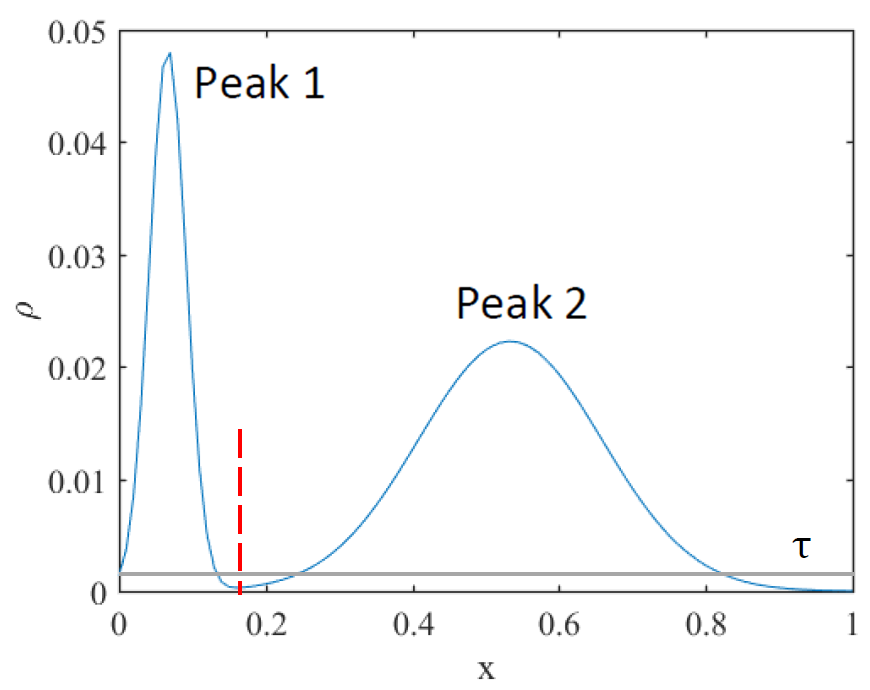}
    \caption{A distribution with two local peaks}
    \label{illud2:a}
  \end{subfigure}  %
  \begin{subfigure}[b]{0.44\textwidth}
  \centering\captionsetup{width=.9\linewidth}%
    \includegraphics[width=\textwidth]{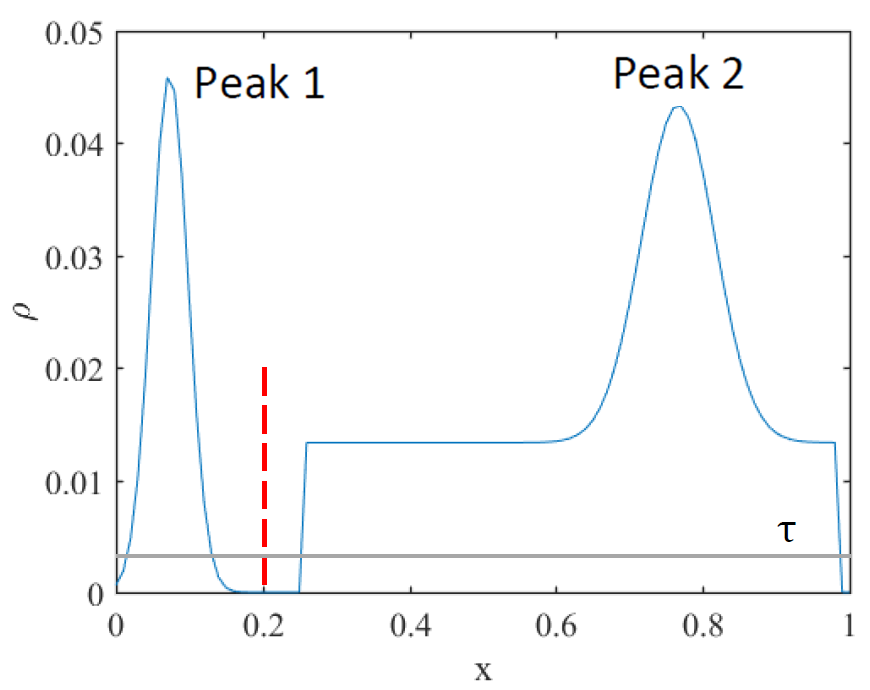}
    \caption{A distribution with two local peaks}
    \label{illud2:b} 
  \end{subfigure} \\
  \begin{subfigure}[b]{0.44\textwidth}
    \centering\captionsetup{width=.9\linewidth}%
    \includegraphics[width=\textwidth]{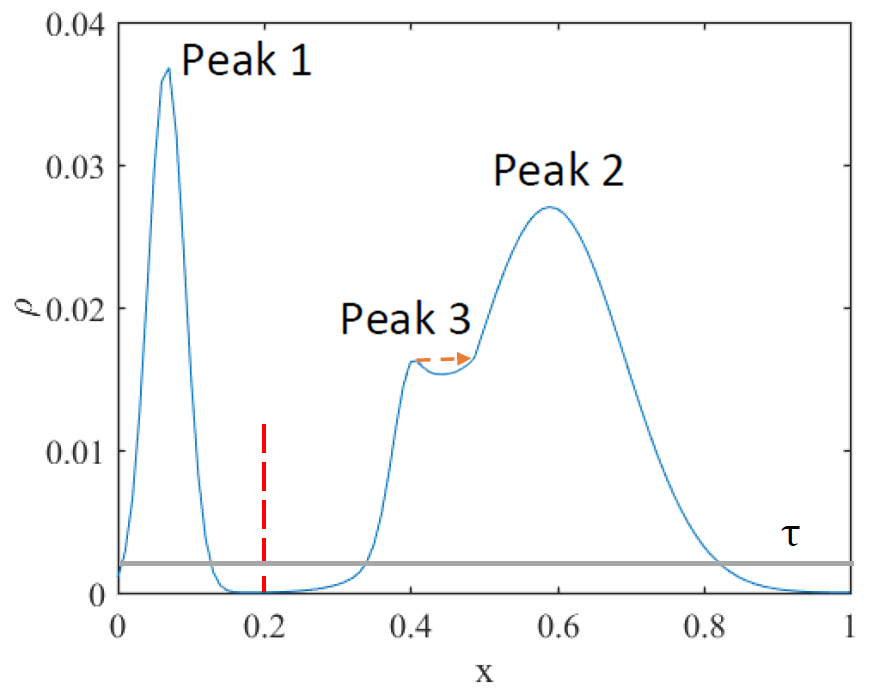}
    \caption{A distribution with three local peaks}
        \label{illud2:c}
  \end{subfigure}
      \begin{subfigure}[b]{0.44\textwidth}
        \centering\captionsetup{width=.9\linewidth}%
    \includegraphics[width=\textwidth]{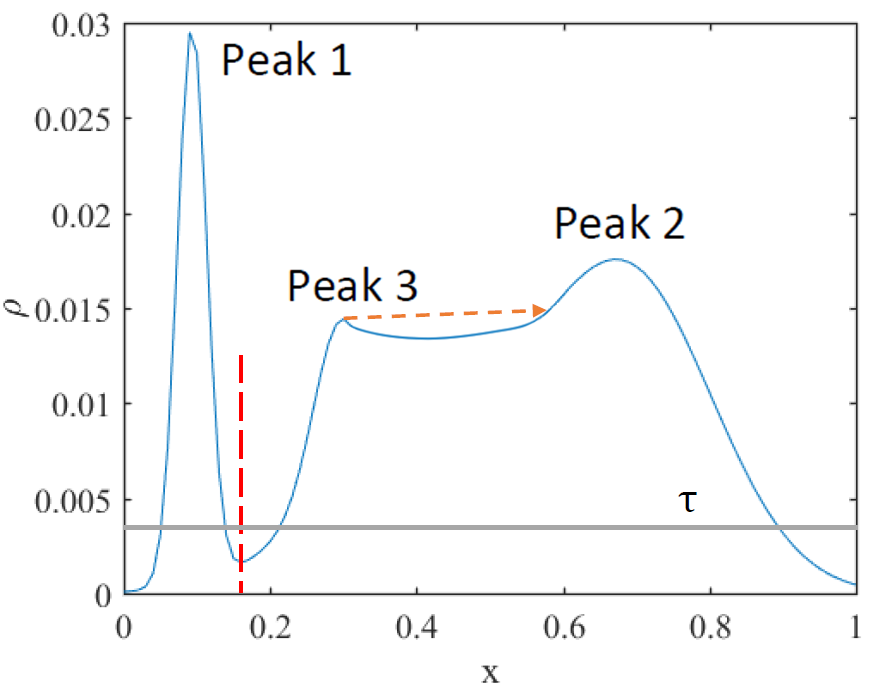}
    \caption{A distribution with three local peaks}
        \label{illud2:d}
  \end{subfigure}
   \caption{Density distribution of four one-dimensional datasets illustrating $\eta$-density-connected clusters. Parameter $\tau$ is the density threshold used for density-connectivity check. The example in (d) shows that all points which are $\eta$-density-connected to Peak 1 are not density-connected with any points which are $\eta$-density-connected to Peak 3. The path between Peak 1 and Peak 3 is disconnected by the $\tau$ setting shown in (d). In (c) and (d), the orange dash-lines point to the nearest density-connected neigbour with higher density of Peak 3. The red dash-line indicates the boundary of the two clusters.}
    \label{illud2} 
\end{figure}

\section{An \texorpdfstring{$\eta$-}{}density-connected  hierarchical clustering algorithm}
\label{DCHDP}
 
Here we propose an $\eta$-density-connected  hierarchical clustering algorithm. It is described in two subsections. In Section~\ref{sec_hierarchical_DP}, we introduce a different view of the (flat) DP as a hierarchical procedure.
Instead of employing a decision graph to rank points proposed in the original DP paper \cite{rodriguez2014clustering}, the proposed hierarchical procedure merges clusters bottom-up to produce a dendrogram. The dendrogram enables a user to identify clusters in a hierarchical way, which cannot be produced by the current flat DP procedure. 

In Section \ref{sec_DC-HDP}, we describe how the hierarchical DP procedure is modified to identify $\eta$-density-connected clusters based on Definition \ref{def:DCHDP}.

\subsection{A different view of DP: a hierarchical procedure}
\label{sec_hierarchical_DP} 
We show that the DP clustering \cite{rodriguez2014clustering} can be accomplished as a hierarchical clustering; and the two clustering procedures produce exactly the same flat clustering result when the same $k$ is used. If we run DP $n$ times by setting $k=n,n-1,...,1$, we get a bottom-up based clustering result. To avoid running DP $n$ times, which has the time complexity of $\mathcal{O}(dn^{3})$, we propose a hierarchical procedure as follows. 

The initialisation step in the hierarchical DP is as follows. After calculating $\gamma$ for all points, let every point $x \in D$ be a cluster mode (which is equivalent to running DP with $k=n$); and each cluster mode is tagged with its $\gamma$.  Let $\mathcal D$ be  $D \setminus \hat{m}$ which is the set used for merging in the next step. 

The first merging of two clusters (which is equivalent to running DP with $k=n-1$) is conducted as follows. Select the cluster having the mode point $z$ with the smallest $\gamma$ value in $\mathcal D$; and the cluster is merged with the cluster having $\eta_{z}$. $z$ is then removed from $\mathcal D$. The above merging process is repeated iteratively by merging two clusters at each iteration until $\mathcal D = \emptyset$.

Figure \ref{fig2} illustrates the  clustering results produced from the hierarchical DP as  dendrograms on the three datasets used in Figures \ref{fig01} and \ref{fig1}. Figure \ref{fig2:a} shows that the elongated cluster is split at the top level in the dendrogram. The dendrogram in Figure \ref{fig2:c} shows that points from the two circles are (incorrectly) merged at low levels in the hierarchical structure. Figure \ref{fig2:e} illustrates that many points from the sparse-and-large cluster are linked to the dense-and-small cluster when $\gamma \approx 0.1$. The flat clustering results extracted from the dendrograms produced by the hierarchical DP are the same as those produced by the flat DP shown in Figures~\ref{fig01} and \ref{fig1}.

\begin{figure}[!htbp]
  \begin{subfigure}[b]{0.44\textwidth}
  \centering\captionsetup{width=.9\linewidth}%
    \includegraphics[width=\textwidth]{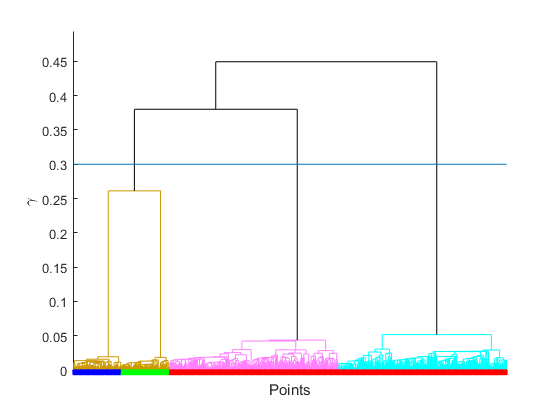}
    \caption{Dendrogram on the 3C dataset}
    \label{fig2:a}
  \end{subfigure}  %
  \begin{subfigure}[b]{0.44\textwidth}
    \centering\captionsetup{width=.9\linewidth}%
    \includegraphics[width=\textwidth]{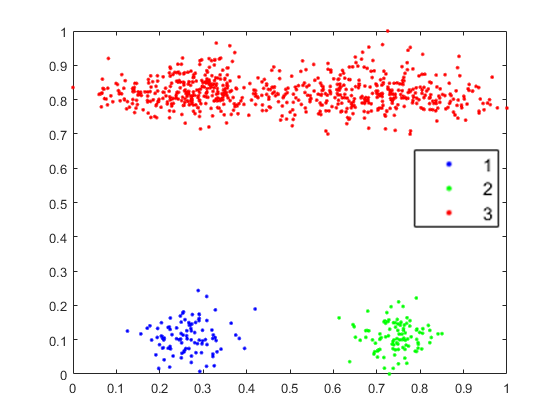}
    \caption{The 3C dataset with true labels}
        \label{fig2:b}
  \end{subfigure}\\
  \begin{subfigure}[b]{0.44\textwidth}
  \centering\captionsetup{width=.9\linewidth}%
    \includegraphics[width=\textwidth]{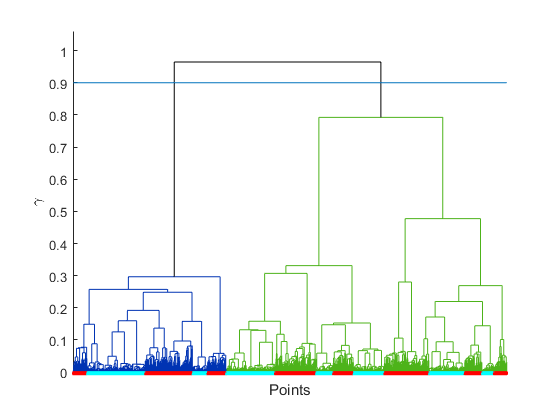}
    \caption{Dendrogram on the 2O dataset}
    \label{fig2:c}
  \end{subfigure}  %
  \begin{subfigure}[b]{0.44\textwidth}
    \centering\captionsetup{width=.9\linewidth}%
    \includegraphics[width=\textwidth]{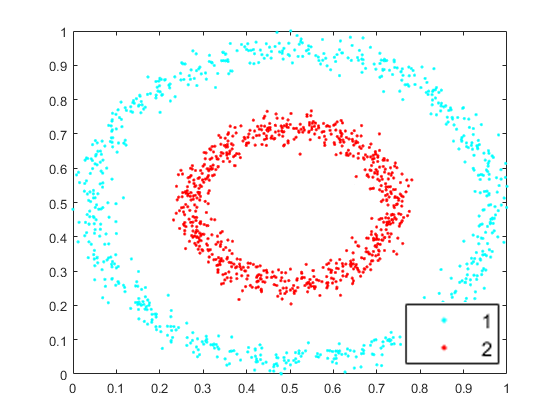}
    \caption{The 2O dataset with true labels}
        \label{fig2:d}
  \end{subfigure}\\
  \begin{subfigure}[b]{0.44\textwidth}
  \centering\captionsetup{width=.9\linewidth}%
    \includegraphics[width=\textwidth]{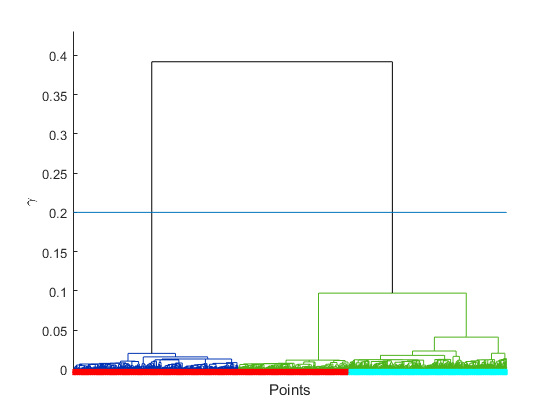}
    \caption{Dendrogram on the 2Q dataset}
    \label{fig2:e}
  \end{subfigure}  %
  \begin{subfigure}[b]{0.44\textwidth}
    \centering\captionsetup{width=.9\linewidth}%
    \includegraphics[width=\textwidth]{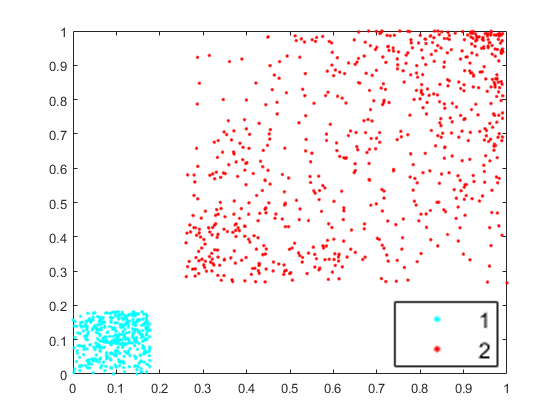}
    \caption{The 2Q dataset with true labels}
        \label{fig2:f}
  \end{subfigure}
   \caption{The hierarchical  DP clustering results on three datasets: they are the same as those produced by the flat DP shown in Figure \ref{fig01} and Figure \ref{fig1}. The horizontal line denotes the $\gamma$ threshold and all points below that line are grouped as a cluster. The colours on the dendrogram branch correspond to the clustering results. The colours at the bottom row in each dendrogram correspond to the true cluster labels of all points shown in Figures (b), (d) and (f).}  
    \label{fig2} 
\end{figure}

Because the cluster modes are selected according to the ranking of $\gamma$, the clustering result with $k$ clusters can be obtained by setting an appropriate threshold of $\gamma$ on the bottom-up based hierarchical clustering result such that the number of clusters below the threshold is $k$. Since both the hierarchical DP and the flat (original) DP  produce the same flat clustering result, the name DP is used hereafter to denote both the two versions, as far as the flat clustering result is concerned.

Here we provide a lemma on the hierarchical view of DP as follows:

\begin{lem}
If there exists exactly one $\eta_x \in D$, $\forall_{x\in D_{\ominus}}$, then the hierarchical view of the DP dendrogram is unique; and the clustering result of DP using a $\gamma$ threshold is unique.
\label{uniView}
\end{lem}

\begin{proof}
Since the dendrogram is built by gradually linking each point $x$ to $\eta_x$ until reaching $\hat{m}$, when $\forall_{x\in D_{\ominus}}$, there exists exactly one $\eta_x \in D$, then the $path(x,\hat{m})$ is unique. Therefore, the hierarchical view of the dendrogram is unique. 

When setting a $\gamma$ threshold, the points with $\gamma$ values higher than the threshold would become cluster modes $m$. Since the other points only have one unique path linking one of the cluster modes and the points linking to the same mode become the same cluster, the clustering result is unique.
\end{proof}

{\bf Advantages of the hierarchical DP}: There are two advantages of the hierarchical DP over the flat DP. First, the former avoids the need to select $k$ cluster modes in the first step of the clustering process. Instead, after the dendrogram is produced at the end of the hierarchical clustering process, $k$ is required only if a flat clustering is to be extracted from the dendrogram.
Second, the dendrogram produced by the hierarchical DP provides a richer information of the hierarchical structure of clusters in a dataset than a flat partitioning provided by the flat DP. 

The hierarchical DP has the same time complexity of the flat DP, i.e.,  $\mathcal{O}(dn^{2})$, since $\gamma$ and $\eta$ are calculated for all points only once. 

\subsection{A density-connected hierarchical DP}
\label{sec_DC-HDP}
In order to enhance  DP to detect clusters from a larger set of data distributions than that covered by density-connected clusters or $\eta$-linked clusters, the clusters based on Definition~\ref{def:DCHDP} is used. 

Using the hierarchical DP, it turns out that only a simple rule needs to be incorporated, i.e., to check whether two cluster modes at the current level are density-connected before merging them at the next level in the hierarchical structure: two clusters $ C_i$ and $ C_j$ can only be merged if there is an $\eta$-density-connected path between their cluster modes.
This is checked at each level of the hierarchy, where the procedure selects the cluster having the mode $x$ with smallest $\gamma'(x)=\rho_{\epsilon}(x) \times \delta'(x)$ to merge with another cluster having $\eta_x'$, where  
 \begin{equation}
     \delta'(x) = \left\{ \begin{array}{l}    
         (i) \ \ dis(x, \eta'_x), \ \ \ \ \ \text{ if} \ \exists_{y\in D} \ y=\eta'_x \\ 
        (ii) \  \displaystyle\max_{y \in D} dis(x, y), \text{ otherwise.}
     \end{array} \right.
 \label{delta2}    
 \end{equation}

We call the new algorithm, DC-HDP, as it employs this cluster merging rule based on the density connectivity with the hierarchical DP procedure. The DC-HDP algorithm is shown in Algorithm \ref{DC-HDP}.

\begin{algorithm}[!htb]
	\caption{DC-HDP($D$, $\epsilon$, $\tau$)}
	\begin{algorithmic}[1] 
		\Require $D$ - input data ($n  \times d$ matrix); $\epsilon$ - radius of the neighbourhood;  $\tau$ - density threshold.
		\Ensure $T$ - a dendrogram (an agglomerative hierarchical cluster tree).
		\State  Calculate $\gamma'(x) =\rho_{\epsilon}(x) \times \delta'(x)$ for each point $x\in D$ based on Equation \ref{epsN} and Equation \ref{delta2}
		\State Initialise $ModeList=\{ m\in D \ | \ \eta'_m \in D \}$, where $\eta'_m$ is based on Equation \ref{eta'}. 
		\State Initialise  $\bar{\mathbb{C}}= \{\langle 0, \bar{C}, \gamma'(x)\rangle \ |\ x \in D \}$; each point in $D$ is a cluster $\bar{C}$. 
		\State $i=1$
		\While {$ModeList \ne \emptyset$} 
		\State Identify the cluster $\bar{C}$ having mode $m \in ModeList$ with the smallest $\gamma'(m)$ 
		\State Identify the cluster $\bar{C_j}$ having the point $\eta'_{m}$ 
		\State Merge clusters $\bar{C}$ and $\bar{C_j}$ as a new cluster $\bar{C_i}$
		\State Add the merged cluster as a triplet $\langle i, \bar{C_i}, \gamma'(m)\rangle$ to $\bar{\mathbb{C}}$ 
		\State Remove $m$ from $ModeList$
		\State $i=i+1$
		\EndWhile
	    \State Merge all points as a single cluster  
	    and add $\langle i, \bar{C_i}, h\rangle$  to  $\bar{\mathbb{C}}$,  where $h=1.1\times max(\gamma'(m))$ 
	    \State Create a dendrogram $T$ based on $\bar{\mathbb{C}}$
		\State \Return $T$ 
	\end{algorithmic}
	\label{DC-HDP} 
\end{algorithm}

Note that, other than $\hat{m}$, there may exist points $x \in D$ which do not have $\eta_x'$.\footnote{An efficient way to find $\eta'_x$ in Equation \ref{delta2} is to search the DBSCAN \cite{ester1996density} cluster which contains $x$.} Let  $G = D \setminus Modelist$. At the end of line 12, only those points which are not $\eta_m'$ for all $m \in Modelist$ become isolated points. 
Also note that it is possible to have more than one cluster at the end of line 12, if the clusters are not density-connected to each other. Therefore, these clusters are merged with all remaining isolated points to yield only one cluster at the top of the hierarchical structure (see line 13 in Algorithm~\ref{DC-HDP}). The $1.1$ in line 13 is only used to make sufficient gaps on the dendrogram.

Once the dendrogram is obtained from Algorithm~\ref{DC-HDP}, a global $\gamma'$ threshold can be used to select the clusters as a flat clustering result. 

Note that the algorithm for the hierarchical DP is the same as the DC-HDP algorithm, except the former uses (i) $\gamma(\cdot)$ instead of $\gamma'(\cdot)$;  and (ii) $D_{\ominus}$ (the whole dataset without the global peak point) instead of $ModeList$. 

Compared with DP, DC-HDP has one more parameter $\tau$, used for the density-connectivity check; and the same $\epsilon$ are used for both density estimation and density connectivity check. DC-HDP maintains the same time complexity of DP, i.e., $\mathcal{O}(dn^{2})$. Similar to hierarchical DP, if $\forall_{x\in G} \exists! \eta'_x \in D$, the hierarchical view of the DC-HDP dendrogram is unique and the clustering result of DC-HDP using a $\gamma'$ threshold is unique.

DC-HDP has the ability to enhance the clustering performance of DP on a dataset having  $\eta$-density-connected clusters which encompass the two kinds of clusters DP is weak at, mentioned in Section~\ref{sec_problem}. This is because DC-HDP does not establish any DCpath between points from different clusters which are not density-connected. Since clusters which are not density-connected are only merged at the top of the dendrogram with the highest $\gamma'$ value, a global $\gamma'$ threshold can separate all these clusters. Unlike DBSCAN, DC-HDP does not rely on a global density threshold to link points; thus, DC-HDP has the ability to detect clusters with varied densities. 

In a nutshell, {\em DC-HDP takes advantage of the individual strengths of DBSCAN and DP, i.e., it has the enhanced ability to identify all clusters of arbitrary shapes and varied densities; where neither DBSCAN nor DP has}.

Figure \ref{fig3} illustrates a clustering result from DC-HDP as a dendrogram on each of the three datasets used in Figure~\ref{fig01} and Figure \ref{fig1}. It shows that all clusters can be detected  perfectly by DC-HDP when an appropriate $\gamma'$ threshold (blue horizontal line) is used on the dendrogram. 

\begin{figure}[!htbp]
  \begin{subfigure}[b]{0.44\textwidth}
  \centering\captionsetup{width=.9\linewidth}%
    \includegraphics[width=\textwidth]{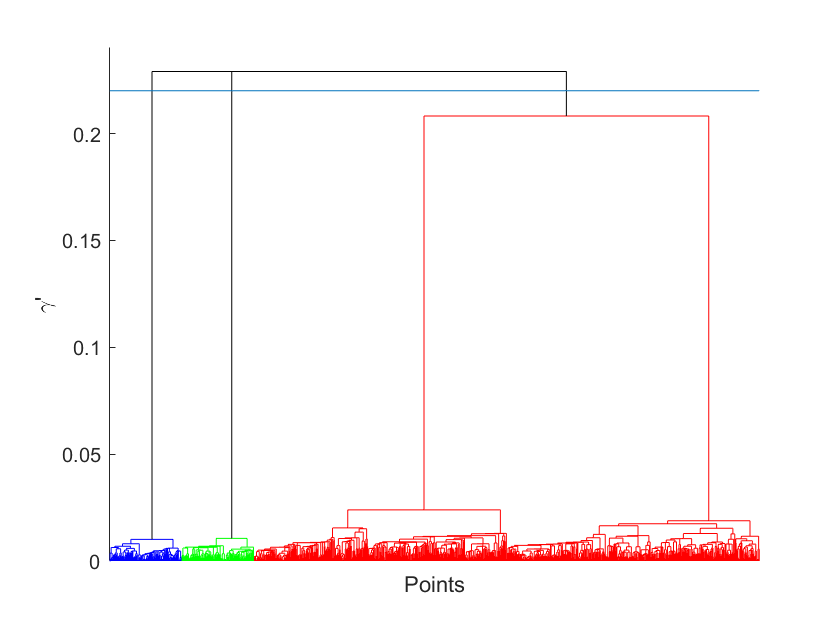}
    \caption{Dendrogram on 3C dataset}
    \label{fig3:f}
  \end{subfigure}  %
  \begin{subfigure}[b]{0.44\textwidth}
    \centering\captionsetup{width=.9\linewidth}%
    \includegraphics[width=\textwidth]{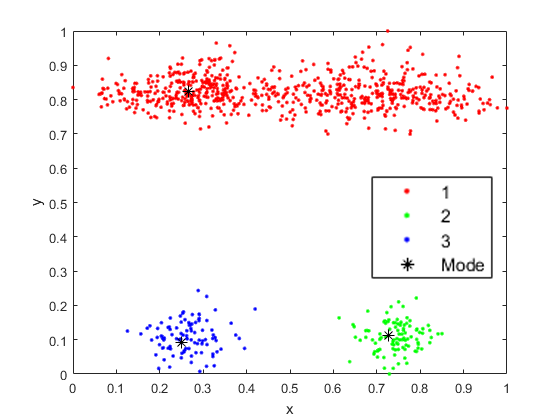}
    \caption{Clustering result on 3C dataset}
        \label{fig3:e}
  \end{subfigure}\\
  \begin{subfigure}[b]{0.44\textwidth}
  \centering\captionsetup{width=.9\linewidth}%
    \includegraphics[width=\textwidth]{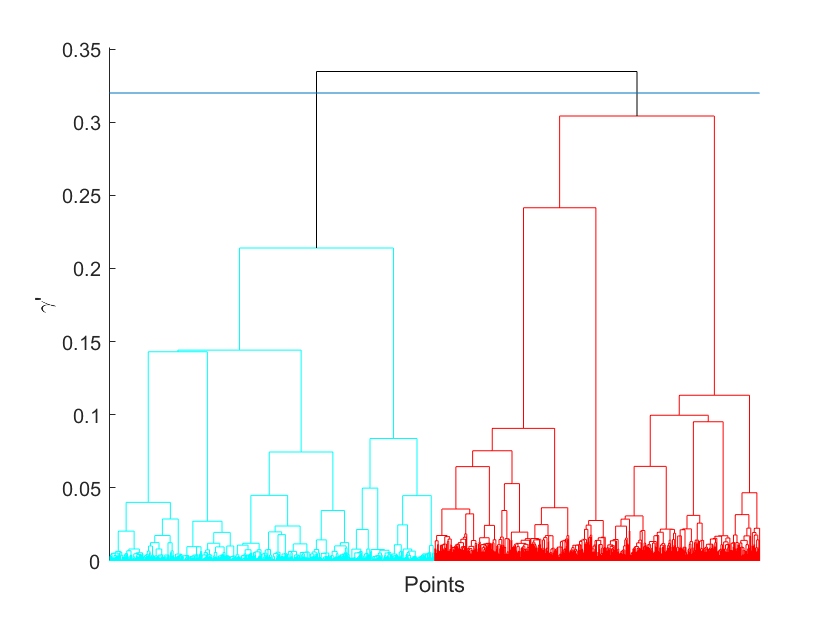}
    \caption{Dendrogram on 2O dataset}
    \label{fig3:a}
  \end{subfigure}  %
  \begin{subfigure}[b]{0.44\textwidth}
    \centering\captionsetup{width=.9\linewidth}%
    \includegraphics[width=\textwidth]{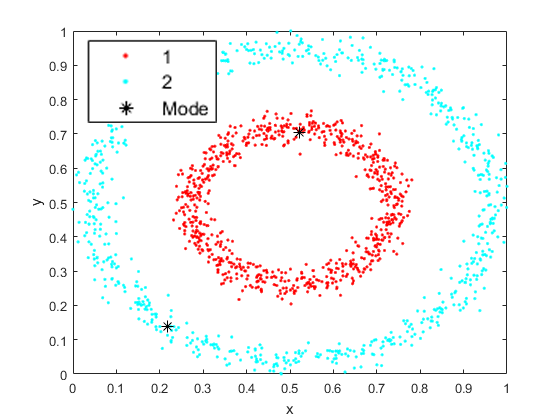}
    \caption{Clustering result on 2O dataset}
        \label{fig3:b}
  \end{subfigure}\\
  \begin{subfigure}[b]{0.44\textwidth}
  \centering\captionsetup{width=.9\linewidth}%
    \includegraphics[width=\textwidth]{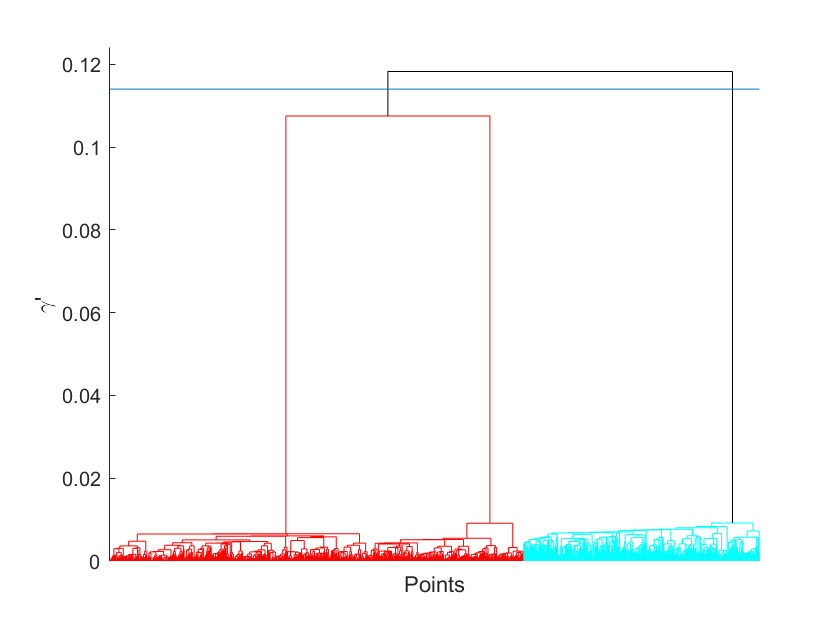}
    \caption{Dendrogram on 2Q dataset}
    \label{fig3:c}
  \end{subfigure}  %
  \begin{subfigure}[b]{0.44\textwidth}
    \centering\captionsetup{width=.9\linewidth}%
    \includegraphics[width=\textwidth]{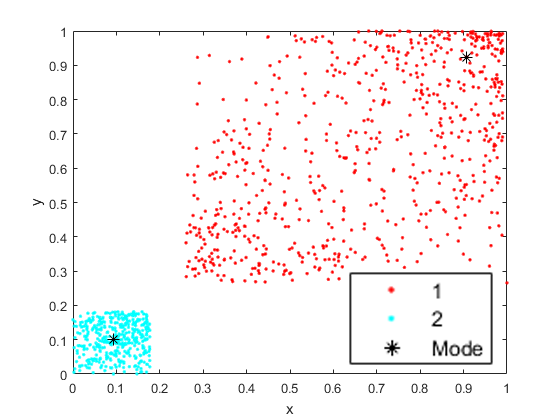}
    \caption{Clustering result on 2Q dataset}
        \label{fig3:d}
  \end{subfigure}
   \caption{DC-HDP's hierarchical clustering results on the three datasets. Setting $\epsilon=0.1$ and $\tau=1$ (with  $k$ set to the true number of clusters) produces the perfect results on all three datasets. The colours used in each dendrogram  (on the left) are the same as used in the corresponding scatter plot (on the right).}
    \label{fig3} 
\end{figure}

Furthermore, DC-HDP has an additional advantage in comparison with DP and DBSCAN, i.e., the dendrogram produced has a rich structure of clusters at different levels. This is the advantage of a hierarchical clustering over a flat clustering.

\section{Empirical evaluation}
\label{sec_result}

This section presents experiments designed to evaluate the effectiveness of DC-HDP. We compare DC-HDP with 4 density-based clustering algorithms (DBSCAN \cite{ester1996density}, Mean shift clustering \cite{comaniciu2002mean},  DP \cite{rodriguez2014clustering} and LC-DP \cite{Chen2018}), 3 hierarchical clustering algorithms (OPTICS \cite{ankerst1999optics}, PHA \cite{LU20131227} and HDBSCAN \cite{Campello:2015:HDE}) and 1 graph-based spectral clustering algorithm \cite{Chen11}. Because LC-DP is an improvement over DP, to conduct a head-to-head comparison with LC-DP, we used the same local contrast $LC(x)$ of LC-DP \cite{Chen2018} for DC-HDP to determine $\eta_x$ in density-connected clusters in the following experiments.

Since clustering is an unsupervised learning task, here we used a standard external evaluation method such that first running the clustering algorithm on the whole dataset with particular parameter settings and then comparing the clustering result with the ground truth \cite{han2011data,aggarwal2013data}. Furthermore, density-based clustering algorithms normally are sensitive to parameter settings because the nonparametric density estimator used in these algorithms suffers from boundary bias without specific knowledge about the domain of the data \cite{botev2010kernel}. To obtain a fair comparison, We report the best clustering performance within a range of parameter search for each algorithm. 

The clustering performance is measured in terms of Macro-average F-measure score\footnote{It is worth noting that other evaluation measures such as purity and Normalized Mutual Information (NMI) \cite{strehl2002cluster} only take into account the points assigned to clusters and do not account for noise. A clustering algorithm which assigns the majority of the points to noise may result in a high clustering performance. Thus the F-measure is more suitable than purity or NMI in assessing the clustering performance of density-based clustering, e.g, DBSCAN and OPTICS.}: given a clustering result, we calculate the precision score $P_{i}$ and the recall score $R_{i}$ for each cluster $C_{i}$ based on the confusion matrix, and the F-measure score of $C_{i}$ is the harmonic mean of $P_{i}$ and $R_{i}$. After computing the pairwise F-measure Score, then we use the Hungarian algorithm \cite{kuhn1955hungarian} to search the optimal match between the clustering results and true clusters. The overall F-measure score is the unweighted average over all matched clusters: F-measure$=\frac{1}{k}\sum_{i=1}^{k}\frac{2P_{i}R_{i}}{P_{i}+R_{i}}$. 

We used 6 artificial datasets (Pathbased, Compound, 2O, 3C, 3G and 2Q) and 11 real-world datasets with different data sizes and dimensions.\footnote{Pathbased is from Chang et al. \cite{chang2008robust}, Compound is from Zahn \cite{zahn1971graph}, Shape is from Müller et al. \cite{Muller2009}, COIL20 is from Li et al. \cite{li2016feature} and all other real-world datasets are from UCI Machine Learning Repository \cite{Lichman:2013}.} Table~\ref{dataset} presents the data properties of the datasets. Figure \ref{figS} shows the scatter plots of Pathbased and Compound datasets.

\begin{table}[!htb]
\scriptsize
 \renewcommand{\arraystretch}{1.2}
 \setlength{\tabcolsep}{4.pt}
  \centering
  \caption{Data properties}
    \begin{tabular}{|c|ccc|}
    \hline
    Dataset & Data Size & \#Dimensions & \#Classes  \\
  \hline
    Pathbased & 300   & 2     & 3  \\
    Compound & 399   & 2     & 6 \\
    3C    & 900   & 2     & 3 \\
    2Q    & 1100  & 2     & 2 \\
    3G    & 1500  & 2     & 3 \\
    2O    & 1500  & 2     & 2 \\     \hdashline
    Shape & 160   & 17    & 9 \\
    LSVT  & 126   & 310   & 2 \\
    GPS   & 163   & 6     & 2 \\
    Seeds & 210   & 7     & 3 \\
    Thyroid & 215   & 5     & 3 \\
    Haberman & 306   & 3     & 2 \\
    Ecoli & 336   & 7     & 8 \\
    Liver & 345   & 6     & 2 \\
    Ionosphere  & 351   & 33    & 2 \\
    Control & 600   & 60    & 6 \\
    Breast & 699   & 9     & 2 \\
    Pima  & 768   & 8     & 2 \\
    Mice  & 1080  & 83    & 8 \\
    Messidor & 1151  & 19    & 2 \\
    Banknote & 1372  & 4     & 2 \\
    COIL20 & 1440  & 1024  & 20 \\
    HumanActivity & 1492  & 561   & 6 \\
    Segment & 2310  & 19    & 7 \\
    Gisette & 7000  & 5000  & 2 \\
    Smartphone & 7767  & 561   & 12 \\
    Pendig & 10992 & 16    & 10 \\
    Magic & 19020 & 10    & 2  \\
    \hline
    \end{tabular}%
  \label{dataset}%
\end{table}%

\begin{figure}[!htb]
  \begin{subfigure}[b]{0.44\textwidth}
  \centering\captionsetup{width=.9\linewidth}%
    \includegraphics[width=\textwidth]{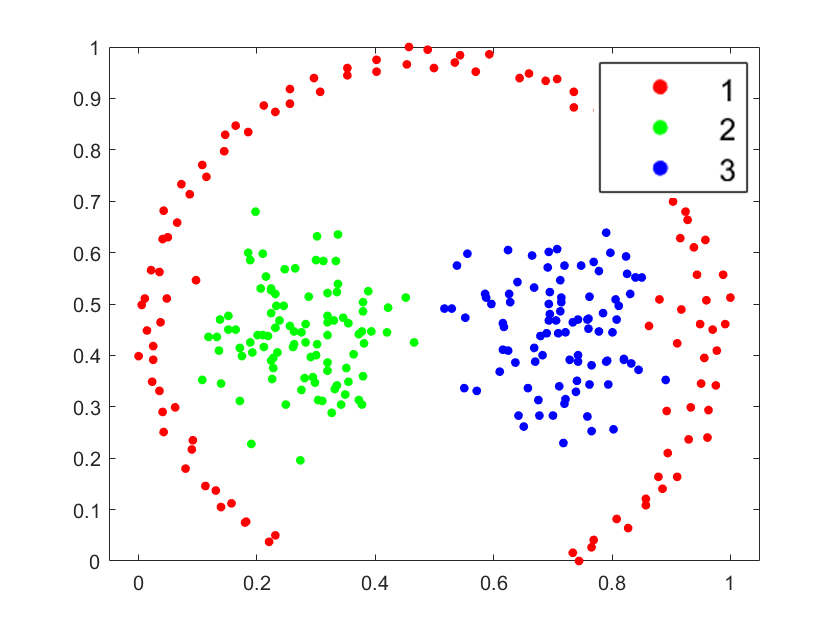}
    \caption{Pathbased dataset}
    \label{figS:a}
  \end{subfigure}  %
  \begin{subfigure}[b]{0.44\textwidth}
    \centering\captionsetup{width=.9\linewidth}%
    \includegraphics[width=\textwidth]{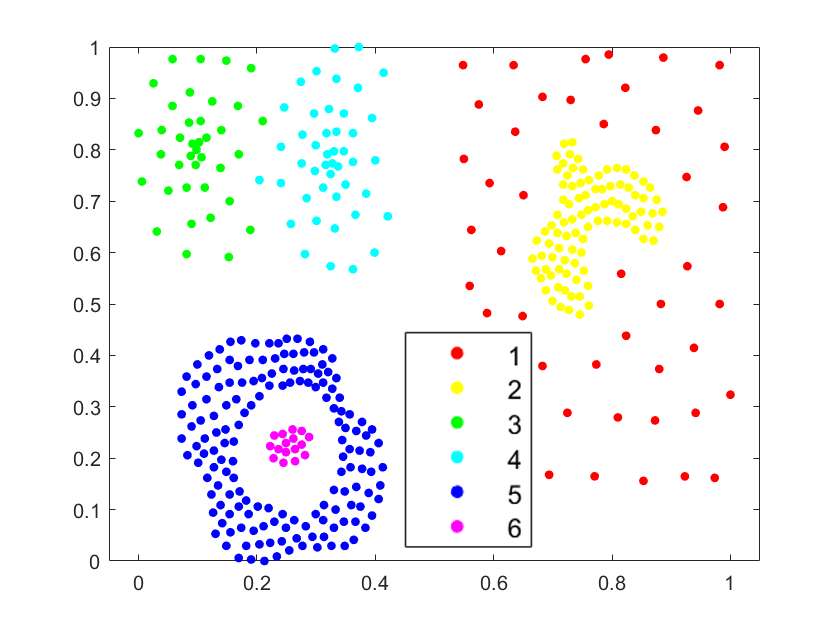}
    \caption{Compound dataset}
        \label{figS:b}
  \end{subfigure}
   \caption{Scatter plots of Pathbased and Compound datasets.}
    \label{figS} 
\end{figure}

All algorithms used in our experiments were implemented in Matlab.\footnote{The source codes of all algorithms used in our experiments can be obtained at:
\vspace{-\topsep}
\begin{itemize}
  \setlength{\parskip}{0pt}
  \setlength{\itemsep}{0pt plus 1pt} 
\item
 {DBSCAN, DC-HDP, LC-DP and DP:} \url{https://sourceforge.net/projects/hierarchical-dp/}
\item
 {HDBSCAN:}  \url{https://au.mathworks.com/matlabcentral/fileexchange/64864-jorsorokin-hdbscan}
\item
 {OPTICS:} \url{https://github.com/alexgkendall/OPTICS_Clustering}
\item
 {Spectral:} \url{https://sites.cs.ucsb.edu/~wychen/sc}
 \item
  {Mean-shift:} \url{https://au.mathworks.com/matlabcentral/fileexchange/10161-mean-shift-clustering}
\item
 {PHA:} \url{https://au.mathworks.com/matlabcentral/fileexchange/46134-fast-hierarchical-clustering-method-pha} 
\end{itemize}} The experiments were run on a machine with eight cores (Intel Core i7-7820X 3.60GHz) and 32GB memory. All datasets were normalised using the $min$-$max$ normalisation to yield each attribute to be in [0,1] before the experiments began.    

For DP, we normalised both $\rho$ and $\delta$ to be in [0,1] before selecting $k$ cluster modes so that these two parameters have the same weight in their product $\rho \times \delta$. Table \ref{para} lists the parameters and their search ranges for each algorithm. Note that the parameter $\xi$ in OPTICS is used to identify downward and upward areas of the reachability plot in order to extract all clusters using a bottom-up hierarchical method \cite{ankerst1999optics}. For all algorithms using $\epsilon$-neighbourhood for density estimation, $\epsilon$ is set to be in $[0.1\%, 99.9\%]$ of the maximum pairwise distance. Note that for DC-HDP, the $k$ parameter is only required to extract $k$ clusters from the dendrogram (at the end of Algorithm \ref{DC-HDP}) by setting a corresponding $\gamma'$ threshold. For the fair comparison with DP, we fix the additional parameter $K=\sqrt{n}$ for LC-DP (as suggested in Chen et al. \cite{Chen2018}). $\tau$ is set to 1 for DC-HDP as the minimum density threshold. 
%

\begin{table}[!htb] 
  \renewcommand{\arraystretch}{1.2}
 \setlength{\tabcolsep}{4.pt}
  \centering
  \caption{Parameters and their search range for each algorithm.}
    \begin{tabular}{|c|c|}
    \hline
    Algorithm & Parameter with search range \\
    \hline
    PHA  & $k \in \{ 2,3,...,50 \}$  \\
    DBSCAN & $Minpts \in \{ 2,3,...,20 \}$; $\epsilon \in \{ 0.1\%,0.2\%,...,99.9\% \}$ \\
    HDBSCAN  &   $Minpts \in \{ 2,3,...,100 \}$; $Minclustsize \in \{ 2,3,...,100 \}$ \\
    OPTICS & $Minpts \in \{ 2,3,...,20 \}$; $\xi \in \lbrace 0.01, 0.02,... , 0.99 \rbrace $\\
    Spectral &  $\sigma\in \{2^{-5},2^{-4},...,2^5 \}$; $k \in \{ 2,3,...,50 \}$  \\
    Mean-shift & $\epsilon \in \{ 0.1\%,0.2\%,...,99.9\% \}$ \\
    DP & $k \in \{ 2,3,...,50 \}$; $\epsilon \in \{ 0.1\%,0.2\%,...,99.9\% \}$ \\
    LC-DP & $k \in \{ 2,3,...,50 \}$; $\epsilon \in \{ 0.1\%,0.2\%,...,99.9\% \}$; $K=\sqrt{n}$  \\
    DC-HDP &  $k \in \{ 2,3,...,50 \}$; $\epsilon \in \{ 0.1\%,0.2\%,...,99.9\% \}$; $\tau =1$ \\
    \hline
    \end{tabular}%
  \label{para}%
\end{table}%

Table \ref{best} shows the best F-measures of the 9 algorithms. In terms of the average F-measure (shown in the third bottom row), DC-HDP has the highest average F-measure of 0.82. The closest contenders are  LC-DP, DP and Spectral clustering which have F-measures: 0.80, 0.76,  and 0.78, respectively.  The next three contenders are OPTICS, DBSCAN and HDBSCAN which have F-measures: 0.72, 0.63 and 0.59, respectively. Mean-shift and PHA have the lowest average F-measures of 0.64 and 0.58, respectively.

  \begin{table}[htbp]
  \scriptsize
    \renewcommand{\arraystretch}{1.2}
  \setlength{\tabcolsep}{2.84pt}
    \centering
     \caption{Best F-measures of different clustering algorithms on 28 datasets. For each clustering algorithm, the best performer in each dataset is boldfaced.}
    \begin{tabular}{|c|ccccccccc|}
    \hline
    Data  & PHA   & DBSCAN & HDBSCAN & OPTICS & Spectral & Mean-shift & DP    & LC-DP & DC-HDP  \\
    \hline
    Pathbased & 0.64  & 0.83  & 0.62  & 0.88  & 0.89  & 0.75  & 0.94  & \textbf{0.96} & \textbf{0.96} \\
    Compound & 0.86  & 0.79  & 0.65  & 0.85  & 0.78  & 0.72  & 0.87  & 0.88  & \textbf{0.94} \\
    3C    & \textbf{1.00} & \textbf{1.00} & \textbf{1.00} & \textbf{1.00} & \textbf{1.00} & \textbf{1.00} & 0.92  & \textbf{1.00} & \textbf{1.00} \\
    2Q    & 0.84  & \textbf{1.00} & \textbf{1.00} & \textbf{1.00} & \textbf{1.00} & \textbf{1.00} & 0.97  & \textbf{1.00} & \textbf{1.00} \\
    3G    & 0.92  & 0.67  & 0.92  & 0.98  & \textbf{0.99} & 0.97  & \textbf{0.99} & \textbf{0.99} & \textbf{0.99} \\
    2O    & 0.51  & \textbf{1.00} & \textbf{1.00} & \textbf{1.00} & \textbf{1.00} & 0.62  & 0.68  & 0.91  & \textbf{1.00}\\
    \hdashline
    Shape & 0.68  & 0.60  & 0.65  & 0.67  & 0.73  & 0.68  & 0.77  & 0.78  & \textbf{0.80}  \\
    LSVT  & 0.40  & 0.38  & 0.40  & 0.70  & 0.56  & 0.43  & 0.68  & 0.73  & \textbf{0.74} \\
    GPS   & 0.75  & 0.75  & 0.68  & 0.76  & 0.81  & 0.76  & 0.81  & \textbf{0.84} & 0.83  \\
    Seeds & 0.90  & 0.71  & 0.50  & 0.80  & 0.93  & 0.91  & 0.91  & \textbf{0.93} & \textbf{0.93} \\
    Thyroid & 0.55  & 0.58  & 0.45  & 0.59  & \textbf{0.96} & 0.61  & 0.87  & 0.87  & 0.90  \\
    Haberman & 0.52  & 0.47  & 0.42  & \textbf{0.63} & 0.52  & 0.55  & 0.56  & 0.61  & \textbf{0.63} \\
    Ecoli & 0.53  & 0.34  & 0.17  & 0.44  & 0.69  & 0.71  & 0.66  & 0.70  & \textbf{0.76} \\
    Liver & 0.44  & 0.37  & 0.37  & \textbf{0.66} & 0.54  & 0.40  & 0.58  & 0.59  & 0.60  \\
    Ionosphere & 0.37  & 0.50  & 0.39  & 0.77  & 0.73  & 0.45  & 0.77  & 0.79  & \textbf{0.91} \\
    Control & 0.50  & 0.53  & 0.60  & 0.64  & 0.77  & 0.62  & 0.74  & 0.77  & \textbf{0.80} \\
    Breast & 0.43  & 0.82  & 0.72  & 0.84  & \textbf{0.97} & \textbf{0.97} & \textbf{0.97} & 0.96  & 0.96  \\
    Pima  & 0.44  & 0.43  & 0.43  & \textbf{0.65} & 0.64  & 0.49  & 0.62  & \textbf{0.65} & \textbf{0.65} \\
    Mice  & 0.99  & 0.99  & 0.99  & 0.98  & \textbf{1.00} & 0.99  & \textbf{1.00} & \textbf{1.00} & \textbf{1.00} \\
    Messidor & 0.50  & 0.48  & 0.48  & 0.62  & 0.58  & 0.38  & 0.55  & 0.53  & \textbf{0.59} \\
    Banknote & 0.96  & 0.95  & 0.89  & 0.95  & 0.99  & 0.77  & \textbf{1.00} & 0.97  & 0.96  \\
    COIL20 & 0.46  & \textbf{0.84} & 0.77  & 0.81  & 0.76  & 0.56  & 0.78  & 0.77  & 0.79  \\
    HumanActivity & 0.19  & 0.33  & 0.30  & 0.37  & 0.57  & 0.24  & \textbf{0.59} & 0.57  & 0.58  \\
    Segment & 0.58  & 0.59  & 0.65  & 0.69  & 0.73  & 0.67  & 0.78  & 0.78  & \textbf{0.79} \\
    Gisette & 0.34  & 0.33  & 0.33  & 0.34  & 0.71  & 0.33  & 0.51  & 0.81  & \textbf{0.83} \\
    Smartphone & 0.11  & 0.19  & 0.10  & 0.29  & \textbf{0.53} & 0.17  & 0.40  & 0.51  & \textbf{0.53} \\
    Pendig & 0.54  & 0.70  & 0.69  & 0.72  & \textbf{0.85} & 0.83  & 0.79  & 0.84  & 0.83  \\
    Magic & 0.40  & 0.42  & 0.42  & 0.65  & 0.65  & 0.40  & 0.66  & \textbf{0.74} & 0.71  \\
    \hline
    \textit{Average} & 0.58  & 0.63  & 0.59  & 0.72  & 0.78  & 0.64  & 0.76  & 0.80  & 0.82  \\
    \hline
    \textit{\# Best} & 1     & 4     & 3     & 6     & 9     & 3     & 5     & 9     & 19  \\
    \hline
    \textit{Average rank} & 7.3   & 6.8   & 7.5   & 4.7   & 3.6   & 6.2   & 3.9   & 2.9   & 2.2  \\
    \hline
    \end{tabular}%
  \label{best}%
  \end{table}%
  
Among the 9 algorithms, both DC-HDP was the top performers on 19 out of 28 datasets, followed by LC-DP and Spectral clustering on 9 datasets. OPTICS and DP were the top performers on 6 and 5 datasets, respectively. DBSCAN, HDBSCAN and Mean-shift were the top performers on 4, 4, and 3 datasets, respectively. 

Notably on the four synthetic datasets (3C, 2Q, 3G and 2O), only DC-HDP, OPTICS and Spectral identified the clusters almost perfectly; while other algorithms failed on at least one of these datasets. Spectral clustering performs poor on the Pathbased and Compound datasets in which  clusters of different densities are closed to each other. LC-DP could not obtain the perfect result on the 2O datasets because it failed to separate the 2 circles. For the 3G dataset which has clusters with varied densities, DBSCAN obtained the lowest F-measure of 0.67. HDBSCAN achieved  F-measure of 0.92 only on this dataset since it assigned many high-density points to noise.

. 

It is interesting to mention that both LC-DP and DC-HDP use a local density estimator which enhances the mode selection of DP. 
The key difference between LC-DP and DC-HDP is that DC-HDP has the density-connectivity check when linking points. Thus, DC-HDP overcomes the drawbacks of both DP and LC-DP in detecting $\eta$-density-connected clusters on the 2O dataset.



To evaluate whether the performance difference among the algorithms is significant, we conduct the Friedman test with the post-hoc Nemenyi test \cite{demvsar2006statistical}. Figure \ref{sig2} shows the significance test result on the 5 algorithms with average F-measure higher than 0.70. It shows that DC-HDP is significantly better than Spectral clustering, DP and OPTICS.

 \begin{figure}[!htbp]
  	\centering
  \begin{subfigure}[b]{0.45\textwidth}
  \centering
    \includegraphics[width=2.3in]{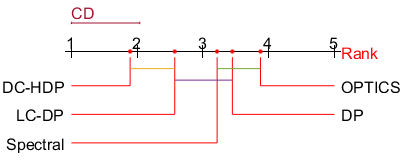}
  \end{subfigure}
  	\caption{Critical difference (CD) diagram of the post-hoc Nemenyi test ($\alpha=0.10$) for 5 clustering algorithms. Two algorithms are significant different if the gap between their ranks is larger than CD. Otherwise, there is a line linking them.}
  	\label{sig2}
  \end{figure}

Table \ref{runtime} presents the runtimes of the 9 algorithms on 4 datasets with different sizes. For a fair comparison, we converted DP and LC-DP to the same hierarchical version as DC-HDP. It shows that DC-HDP is only a bit slower than DP in practice due to the additional density connectivity check. Note that Spectral clustering has $\mathcal{O}(n^3)$ but others have $\mathcal{O}(n^2)$ time complexity.

\begin{table*}[!htbp]
   \setlength{\tabcolsep}{2.84pt}
  \centering
  \caption{Runtime (in seconds) comparison}   
      \begin{tabular}{|c|ccccccccc|}
      \hline
      Data  & PHA   & DBSCAN & HDBSCAN & OPTICS & Spectral & Mean-shift & DP    & LC-DP & DC-HDP\\
      \hline
      Banknote & 0.05  & 0.02  & 1.15  & 0.10  & 0.40  & 0.02  & 0.04  & 0.11  & 0.06  \\
      Segment & 0.17  & 0.06  & 1.94  & 0.26  & 4.91  & 0.31  & 0.10  & 0.30  & 0.14  \\
      Pendig & 5.39  & 0.87  & 15.78  & 4.96  & 105.75  & 82.65  & 2.34  & 5.66  & 6.01  \\
      Magic & 16.01  & 2.94  & 55.59  & 15.49  & 154.76  & 51.42  & 6.79  & 15.15  & 10.40 \\
      \hline
      \end{tabular}%
  \label{runtime}%
\end{table*}%

\section{Discussion}
 
\subsection{Relation to existing hierarchical clustering algorithms}

Among existing hierarchical clustering methods, DC-HDP is closest to traditional agglomerative methods, where they all employ the bottom-up strategy to merge two individual clusters at each level successively in a tree structure.

However, DC-HDP and hierarchical DP differ from traditional agglomerative methods in two ways.
First, when a traditional agglomerative method merges two clusters, they are selected based on a set-based dissimilarity measure. There are different measures that can be used to merge two clusters by trading off quality versus efficiency, such as single-linkage, complete-linkage, all-pairs linkage and centroid-linkage measures \cite{aggarwal2013data}.
In contrast, DC-HDP and hierarchical DP do not simply employ a dissimilarity measure to determine the two clusters to merge at each level. Instead, they first identify the cluster having the smallest $\gamma'$ (and $\gamma$), and then select another cluster which has the shortest path length to it. While the path length may be considered as a kind of dissimilarity measure, that is a supporting measure, and the key determinant is $\gamma'$.

Second, different from traditional methods, DC-HDP and hierarchical DP are a new agglomerative approach detecting $\eta$-density-connected clusters and $\eta$-linked clusters, respectively. Therefore, they can detect arbitrarily shaped clusters while existing agglomerative methods generally detect clusters with specific shapes, e.g, single-linkage measure tends to output elongated-shaped clusters, complete-linkage measure tends to detect compact-shaped clusters, all-pairs linkage and centroid-linkage measures tend to find globular clusters \cite{aggarwal2013data}.

The standard algorithm for hierarchical agglomerative clustering normally has a time complexity of $\mathcal{O}(n^{3})$ \cite{Day1984}. However, many efficient hierarchical agglomerative clustering approaches have the same quadratic time complexity and space complexity as DC-HDP when the pairwise distance matrix is require as input, e.g., SLINK \cite{Sibson1973} for single-linkage and CLINK \cite{SLINK1977} for complete-linkage clustering. 


PHA measures the similarity between two clusters based on a hypothetical potential field that relies on both local and global data distribution information. It can detect slightly overlapping clusters with non-spherical shapes in noisy data. However, compared with density-based methods (e.g., DBSCAN, HDBSCAN, OPTICS and DC-HDP), PHA performed much worse in detecting arbitrarily shaped clusters, e.g., on the 2Q and 2O datasets.

There is another class of algorithms which employs a method to produce an initial set of subclusters from data, before applying a hierarchical clustering. For example,
CHAMELEON \cite{781637} produces a $K$-nearest-neighbour graph from data and then breaks the graph into many small subgraphs (as subclusters). An agglomerative method is finally used to merge subclusters iteratively based on a similarity measure. 
The same general approach is used in two more recent methods, i.e., HDBSCAN \cite{Campello:2015:HDE} and OPTICS \cite{ankerst1999optics}; though different methods are used to produce subclusters in the preprocessing before building a hierarchical structure on them.

We show that DC-HCP is a simple yet effective approach than this class of algorithms because DC-HDP applies agglomerative clustering directly on individual points in the given dataset without a preprocessing to create subclusters. Section \ref{sec_result} shows that DC-HDP produces a significantly better clustering result than the most recent representative of this class of algorithms, i.e., HDBSCAN, as well as OPTICS.

\subsection{Parameter settings}
DC-HDP requires two parameters $\epsilon$ and $\tau$ to build a dendrogram from a dataset, as shown in Table \ref{para} where $\epsilon$ is more important as it is used in both density estimation and density connectivity check. In our experiments, we found that $\tau$ can be set to 1 (i.e., at least 1 point in the $\epsilon$-neighbourhood) in most datasets in terms of getting the best clustering results. 

In all empirical evaluations reported in Section \ref{sec_result}, we used the same $\epsilon$ parameter for both the density estimation and density connectivity check for DC-HDP. However, we can split $\epsilon$ into two different $\epsilon$ parameters for the two processes individually. By doing so, we found that DC-HDP can perform even better than the results shown in Table \ref{best} on some datasets.

\subsection{Ability to detect noise}
It is worth mentioning that density-based clustering has the ability to identify noise and then filter them out in clustering. For example, DBSCAN uses a global density threshold to identify noise as points with a density lower than the threshold  in the first step of the algorithm \cite{ester1996density}. DP \cite{rodriguez2014clustering} employs a different method to identify the noise which are points with low densities at border regions of clusters (see footnote \ref{foootnote_dp} in Section \ref{sec_related_work} for details). This is conducted at the end of the clustering process. This same method can be used by DC-HDP to identify noise. 

\section{Conclusions} 
The lack of a cluster definition, that a state-of-the-art density-based algorithm called Density Peak (DP) can detect, has motivated the work in this paper.

We formally defined two new kinds of clusters: $\eta$-linked clusters and $\eta$-density-connected clusters. A further analysis revealed that DP is a clustering algorithm detecting $\eta$-linked clusters; and it has weaknesses in data distributions which contain a special kind of $\eta$-linked clusters or some non-$\eta$-linked clusters.  We show that $\eta$-density-connected clusters encompass all $\eta$-linked clusters and the kind of non-$\eta$-linked clusters that DP fails to detect.

After showing that DP clustering can be accomplished as a hierarchical clustering,  we proposed a density-connected hierarchical DP clustering called DC-HDP, which is designed to detect $\eta$-density-connected clusters.

By taking advantage of the individual strengths of DBSCAN and DP, DC-HDP produces clustering outputs which are superior in two key aspects. First, DC-HDP has an enhanced ability to identify clusters of arbitrary shapes and varied densities; where neither DBSCAN nor DP has. Second, the dendrogram generated by DC-HDP gives a richer information of the hierarchical structure of clusters in a dataset than a flat partitioning provided by DBSCAN and DP. DC-HDP achieved the enhanced ability with the same time complexity as DP. The additional parameters of DC-HDP can be set to default values in practice.  

We confirm the previous study that LC-DP is an improvement over DP; and show that the proposed DC-HDP is further improvement over both LC-DP and DP. Our contribution is not merely algorithmic improvement, but formal cluster definitions which were non-existence in the previous studies. These formal definitions are the foundation of DC-HDP.

Our empirical evaluation validates this superiority by showing that DC-HDP produces the best clustering results on 28 datasets in comparison with 8  state-of-the-art clustering algorithms, including density-based clustering, i.e., DBSCAN, Mean Shift clustering, DP and LC-DP; hierarchical clustering, i.e., HDBSCAN, PHA and OPTICS; and graph-based spectral clustering algorithm.

 

\bibliography{ref}

\end{document}